\documentclass[11pt]{article}
\usepackage{fullpage}
\usepackage{times}
\usepackage{hyperref}
\usepackage{url}
\usepackage{sidecap}
\usepackage{amsmath,amssymb,amsthm,graphicx,url}
\usepackage{algorithm,algorithmic}

\newtheorem{theorem}{Theorem}

\newtheorem{lemma}{Lemma}

\newtheorem{definition}{Definition}

\newcommand{\reals}{\mathbb{R}}
\newcommand{\E}{\mathbb{E}}

\newcommand{\be}{\mathbf{e}}
\newcommand{\bx}{\mathbf{x}}
\newcommand{\bw}{\mathbf{w}}

\newcommand{\bu}{\mathbf{u}}
\newcommand{\bv}{\mathbf{v}}
\newcommand{\bz}{\mathbf{z}}

\newcommand{\bh}{\mathbf{h}}

\newcommand{\Ocal}{\mathcal{O}}

\newcommand{\Xcal}{\mathcal{X}}

\newcommand{\secref}[1]{Sec.~\ref{#1}}
\newcommand{\subsecref}[1]{Subsection~\ref{#1}}

\renewcommand{\eqref}[1]{Eq.~(\ref{#1})}
\newcommand{\lemref}[1]{Lemma~\ref{#1}}
\newcommand{\thmref}[1]{Thm.~\ref{#1}}

\title{Fundamental Limits of Online and Distributed Algorithms
for Statistical Learning and Estimation}
\date{}

\author{Ohad Shamir\\Weizmann Institute of Science\\\texttt{ohad.shamir@weizmann.ac.il}}

\begin{document}

\maketitle

\begin{abstract}
Many machine learning approaches are characterized by information
constraints on how they interact with the training data. These include
memory and sequential access constraints (e.g. fast first-order methods to
solve stochastic optimization problems); communication constraints (e.g.
distributed learning); partial access to the underlying data (e.g. missing
features and multi-armed bandits) and more. However, currently we have
little understanding how such information constraints fundamentally affect
our performance, independent of the learning problem semantics. For
example, are there learning problems where any algorithm which has small
memory footprint (or can use any bounded number of bits from each example,
or has certain communication constraints) will perform worse than what is
possible without such constraints? In this paper, we describe how a single
set of results implies positive answers to the above, for several different
settings.
\end{abstract}

\section{Introduction}

Information constraints play a key role in machine learning. Of course, the
main constraint is the availability of only a finite data set, from which the
learner is expected to generalize. However, many problems currently
researched in machine learning can be characterized as learning with
\emph{additional} information constraints, arising from the manner in which
the learner may interact with the data. Some examples include:
\begin{itemize}
    \item \emph{Communication constraints in distributed learning:} There
        has been much work in recent years on learning when the training
        data is distributed among several machines (with
        \cite{bekkerman2011scaling,agarwal2011reliable,dekel2011optimal,niu2011hogwild,cotter2011better,duchi2012dual,balcan2012distributed,boyd2011distributed,kyrola2011parallel}
        being just a few examples). Since the machines may work in
        parallel, this potentially allows significant computational
        speed-ups and the ability to cope with large datasets. On the flip
        side, communication rates between machines is typically much slower
        than their processing speeds, and a major challenge is to perform
        these learning tasks with minimal communication.
    \item \emph{Memory constraints:} The standard implementation of many
        common learning tasks requires memory which is super-linear in the
        data dimension. For example, principal component analysis (PCA)
        requires us to estimate eigenvectors of the data covariance matrix,
        whose size is quadratic in the data dimension and can be
        prohibitive for high-dimensional data. Another example is kernel
        learning, which requires manipulation of the Gram matrix, whose
        size is quadratic in the number of data points. There has been
        considerable effort in developing and analyzing algorithms for such
        problems with reduced memory footprint (e.g.
        \cite{mitliagkas2013memory,arora2013stochastic,balsubramani2013fast,williams2001using,rahimi2007random}).
    \item \emph{Online learning constraints:} The need for fast and
        scalable learning algorithms has popularised the use of online
        algorithms, which work by sequentially going over the training
        data, and incrementally updating a (usually small) state vector.
        Well-known special cases include gradient descent and mirror
        descent algorithms (see e.g.
        \cite{shalev2011online,sra2011optimization}). The requirement of
        sequentially passing over the data can be seen as a type of
        information constraint, whereas the small state these algorithms
        often maintain can be seen as another type of memory constraint.
    \item \emph{Partial-information constraints:} A common situation in
        machine learning is when the available data is corrupted, sanitized
        (e.g. due to privacy constraints), has missing features, or is
        otherwise partially accessible. There has also been considerable
        interest in online learning with partial information, where the
        learner only gets partial feedback on his performance. This has
        been used to model various problems in web advertising, routing and
        multiclass learning. Perhaps the most well-known case is the
        multi-armed bandits problem
        \cite{bubeck2012regret,auer2002nonstochastic,auer2002finite}, with
        many other variants being developed, such as contextual bandits
        \cite{langford2007epoch,li2010contextual}, combinatorial bandits
        \cite{cesa2012combinatorial}, and more general models such as
        partial monitoring \cite{bubeck2012regret,bartok2013partial}.
\end{itemize}

Although these examples come from very different domains, they all share the
common feature of information constraints on how the learning algorithm can
interact with the training data. In some specific cases (most notably,
multi-armed bandits, and also in the context of certain distributed protocols, e.g.
\cite{balcan2012distributed,ZDJW13}) we can even formalize the price we pay
for these constraints, in terms of degraded sample complexity or regret
guarantees. However, we currently lack a general information-theoretic
framework, which directly quantifies how such constraints can impact
performance. For example, are there cases where any online algorithm, which
goes over the data one-by-one, must have a worse sample complexity than (say)
empirical risk minimization? Are there situations where a small memory
footprint provably degrades the learning performance? Can one quantify how a
constraint of getting only a few bits from each example affects our ability
to learn? To the best of our knowledge, there are currently no generic tools
which allow us to answer such questions, at least in the context of standard
machine learning settings.

In this paper, we make a first step in developing such a framework. We
consider a general class of learning processes, characterized only by
information-theoretic constraints on how they may interact with the data (and
independent of any specific problem semantics). As special cases, these
include online algorithms with memory constraints, certain types of
distributed algorithms, as well as online learning with partial information.
We identify cases where any such algorithm must perform worse than what can
be attained without such information constraints. The tools developed allows
us to establish several results for specific learning problems:
\begin{itemize}
    \item We prove a new and generic regret lower bound for
        partial-information online learning with expert advice. The lower
        bound is $\Omega(\sqrt{(d/b)T})$, where $T$ is the number of
        rounds, $d$ is the dimension of the loss/reward vector, and $b$ is
        the number of bits $b$ extracted from each loss vector. It is
        optimal up to log-factors (without further assumptions), and holds
        no matter what these $b$ bits are -- a single coordinate (as in
        multi-armed bandits), some information on several coordinates
        (studied in various settings including semi-bandit feedback,
        bandits with side observations, and prediction with limited
        advice), a linear projection (as in bandit linear optimization),
        some feedback signal from a restricted set (as in partial
        monitoring) etc. Interestingly, it holds even if the online learner
        is allowed to adaptively choose which bits of the loss vector it
        can retain at each round. The lower bound quantifies in a very
        direct way how information constraints in online learning degrade
        the attainable regret, independent of the problem semantics.
    \item We prove that for some learning and estimation problems - in
        particular, sparse PCA and sparse covariance estimation in
        $\reals^d$ - no online algorithm can attain statistically optimal
        performance (in terms of sample complexity) with less than
        $\tilde{\Omega}(d^2)$ memory. To the best of our knowledge, this is
        the first formal example of a \emph{memory/sample complexity}
        trade-off in a statistical learning setting.
    \item We show that for similar types of problems, there are cases where
        no distributed algorithm (which is based on a non-interactive or
        serial protocol on i.i.d. data) can attain optimal performance with
        less than $\tilde{\Omega}(d^2)$ communication per machine. To the
        best of our knowledge, this is the first formal example of a
        \emph{communication/sample complexity} trade-off, in the regime
        where the communication budget is larger than the data dimension,
        and the examples at each machine come from the same underlying
        distribution.
    \item We demonstrate the existence of simple (toy) stochastic
        optimization problems where any algorithm which uses memory linear
        in the dimension (e.g. stochastic gradient descent or mirror
        descent) cannot be statistically optimal.
\end{itemize}

\subsection*{Related Work}

In stochastic optimization, there has been much work on lower bounds for
sequential algorithms, starting from the seminal work of \cite{YudNem83}, and
including more recent works such as \cite{agarwal2012information}.
\cite{raginsky2011information} also consider such lower bounds from a more
general information-theoretic perspective. However, these results all hold in
an \emph{oracle model}, where data is assumed to be made available in a
specific form (such as a stochastic gradient estimate). As already pointed
out in \cite{YudNem83}, this does not directly translate to the more common
setting, where we are given a dataset and wish to run a simple sequential
optimization procedure. Indeed, recent works exploited this gap to get
improved algorithms using more sophisticated oracles, such as the
availability of prox-mappings \cite{nesterov2005smooth}. Moreover, we are not
aware of cases where these lower bounds indicate a gap between the attainable
performance of any sequential algorithm and batch learning methods (such as
empirical risk minimization).

In the context of distributed learning and statistical estimation,
information-theoretic lower bounds have been recently shown in the pioneering
work \cite{ZDJW13}. Assuming communication budget constraints on different
machines, the paper identifies cases where these constraints affect
statistical performance. Our results (in the context of distributed learning)
are very similar in spirit, but there are two important differences. First,
they pertain to parametric estimation in $\reals^d$, where the communication
budget per machine is much smaller than what is needed to even specify the
answer ($\Ocal(d)$ bits). In contrast, our results pertain to simpler
detection problems, where the answer requires only $\Ocal(\log(d))$ bits, yet
lead to non-trivial lower bounds even when the budget size is much larger (in
some cases, much larger even than $d$). The second difference is that their
work focuses on distributed algorithms, while we address a more general class of algorithms,
which includes other information-constrained settings. Strong lower bounds in
the context of distributed learning have also been shown in
\cite{balcan2012distributed}, but they do not apply to a regime where
examples across machines come from the same distribution, and where the
communication budget is much larger than what is needed to specify the
output.

There are well-known lower bounds for multi-armed bandit problems and other
online learning with partial-information settings. However, they crucially
depend on the semantics of the information feedback considered. For example,
the standard multi-armed bandit lower bound \cite{auer2002nonstochastic}
pertain to a setting where we can view a single coordinate of the loss
vector, but doesn't apply as-is when we can view more than one coordinate (as
in semi-bandit feedback \cite{gyorgy2007line,audibert2011minimax}, bandits
with side observations \cite{mannor2011bandits}, or prediction with limited advice
\cite{seldin2014prediction}), receive a linear projection
(as in bandit linear optimization), or receive a different type of partial
feedback (such as in partial monitoring \cite{cesa2006prediction}). In
contrast, our results are generic and can directly apply to any such setting.

The inherent limitations of streaming and distributed algorithms, including
memory and communication constraints, have been extensively studied within
theoretical computer science (e.g.
\cite{alon1996space,bar2002information,chakrabarti2003near,muthukrishnan2005data,barak2010compress}).
Unfortunately, almost all these results consider tasks unrelated to learning,
and/or adversarially generated data, and thus do not apply to statistical
learning tasks, where the data is assumed to be drawn i.i.d. from some
underlying distribution. \cite{woodruff2009average,cr13} do consider i.i.d.
data, but focus on problems such as detecting graph connectivity and counting
distinct elements, and not learning problems such as those considered here.
On the flip side, there are works on memory-efficient algorithms with formal
guarantees for statistical problems (e.g.
\cite{mitliagkas2013memory,balsubramani2013fast,guha2007space,chien2010space}),
but these do not consider lower bounds or provable trade-offs.

Finally, there has been a line of works on hypothesis testing and statistical
estimation with finite memory (see
\cite{hellman1970learning,leighton1986estimating,ertin2003sequential,kontorovich2012statistical}
and references therein). However, the limitations shown in these works apply
when the required precision exceeds the amount of memory available, a regime
which is usually relevant only when the data size is exponential in the
memory size\footnote{For example, suppose we have $B$ bits of memory and try
to estimate the mean of a random variable in $[0,1]$. If we have $2^{4B}$
data points, then by Hoeffding's inequality, we can estimate the mean up to
accuracy $\Ocal(2^{-2B})$, but the finite memory limits us to an accuracy of
$2^{-B}$.}. In contrast, we do not rely on finite precision considerations.

\section{Information-Constrained Protocols}\label{sec:prelim}

We begin with a few words about notation. We use bold-face letters (e.g.
$\bx$) to denote vectors, and let $\be_j\in \reals^d$ denote $j$-th standard
basis vector. When convenient, we use the standard asymptotic notation
$\Ocal(\cdot),\Omega(\cdot),\Theta(\cdot)$ to hide constants, and an
additional $\tilde{}$ sign (e.g. $\tilde{\Ocal}(\cdot)$) to also hide
log-factors. $\log(\cdot)$ refers to the natural logarithm, and
$\log_2(\cdot)$ to the base-2 logarithm.

Our main object of study is the following generic class of information-constrained
algorithms:
\begin{definition}[$(b,n,m)$ Protocol]\label{def:pro}
Given access to a sequence of $mn$ i.i.d. instances (vectors in $\reals^d$),
an algorithm is a $(b,n,m)$ protocol if it has the following form, for some
functions $f_t$ returning an output of at most $b$ bits, and some function
$f$:
\begin{itemize}
    \item For $t=1,\ldots,m$
    \begin{itemize}
        \item Let $X^t$ be a batch of $n$ i.i.d. instances
        \item Compute message $W^t = f_t(X^t,W^1,W^2,\ldots W^{t-1})$
    \end{itemize}
    \item Return $W=f(W^1,\ldots,W^m)$
\end{itemize}
\end{definition}
Note that the functions $\{f_t\}_{t=1}^{m},f$ are completely arbitrary, may
depend on $m$ and can also be randomized. The crucial assumption is that the
outputs $W^t$ are constrained to be only $b$ bits.

At this stage, the definition above may appear quite abstract, so let us
consider a few specific examples:
\begin{itemize}
  \item $b$-memory online protocols: Consider any algorithm which goes over
      examples one-by-one, and incrementally updates a state vector $W^t$
      of bounded size $b$. We note that a majority of online learning and
      stochastic optimization algorithms have bounded memory. For example,
for linear predictors, most gradient-based algorithms maintain a state
whose size is proportional to the size of the parameter vector that is
being optimized. Such algorithms correspond to $(b,n,m)$ protocols where
$n=1$; $W^t$ is the state vector after round $t$, with an update function
$f_t$ depending only on $W^{t-1}$, and $f$ depends only on $W^m$.
\item \emph{Non-interactive and serial distributed algorithms:} There are
    $m$ machines and each machine receives an independent sample $X^t$ of
    size $n$. It then sends a message $W^t = f_t(X^t)$ (which here depends
    only on $X^t$). A centralized server then combines the messages to
    compute an output $f(W^1\ldots W^m)$. This includes for instance
    divide-and-conquer style algorithms proposed for distributed stochastic
    optimization (e.g. \cite{zhang2012communication,zhang2013divide}). A
    serial variant of the above is when there are $m$ machines, and
    one-by-one, each machine $t$ broadcasts some information $W^t$ to the
    other machines, which depends on $X^t$ as well as previous messages
    sent by machines $1,2,\ldots,(t-1)$.
\item \emph{Online learning with partial information:} This is a special
    case of $(b,1,m)$ protocols. We sequentially receive $d$-dimensional
    loss vectors, and from each of these we can extract and use only $b$
    bits of information, where $b\ll d$. For example, this includes most
    types of multi-armed bandit problems.
\item \emph{Mini-batch Online learning algorithms:} The data is streamed
    one-by-one or in mini-batches of size $n$, with $mn$ instances overall.
    An algorithm sequentially updates its state based on a $b$-dimensional
    vector extracted from each example/batch (such as a gradient or
    gradient average), and returns a final result after all data is
    processed. This includes most gradient-based algorithms we are aware
    of, but also distributed versions of these algorithms (such as
    parallelizing a mini-batch processing step as in
    \cite{dekel2012optimal,cotter2011better}).
\end{itemize}
We note that our results can be generalized to allow the size of the messages
$W^t$ to vary across $t$, and even to be chosen in a data-dependent manner.

In our work, we contrast the performance attainable by \emph{any} algorithm
corresponding to such protocols, to \emph{constraint-free} protocols which
are allowed to interact with the sampled instances in any manner.

\section{Basic Results}\label{sec:basic}

Our results are based on a simple `hide-and-seek' statistical estimation
problem, for which we show a strong gap between the attainable performance of
information-constrained protocols and constraint-free protocols. It is
parameterized by a dimension $d$, bias $\rho$, and sample size $mn$, and
defined as follows:
\begin{definition}[Hide-and-seek Problem]\label{def:prob1}
Consider the set of product distributions $\{\Pr_j(\cdot)\}_{j=1}^{d}$ over
$\{-1,1\}^d$ defined via $\E_{\bx\sim
\Pr_j(\cdot)}[x_i]=2\rho~\mathbf{1}_{i=j}$ for all coordinates
$i=1,\ldots d$. Given an i.i.d. sample of $mn$ instances generated from
$\Pr_j(\cdot)$, where $j$ is unknown, detect $j$.
\end{definition}
In words, ${\Pr}_j(\cdot)$ corresponds to picking all coordinates other than
$j$ to be $\pm 1$ uniformly at random, and independently picking coordinate
$j$ to be $+1$ with a higher probability $\left(\frac{1}{2}+\rho\right)$. The
goal is to detect the biased coordinate $j$ based on a sample.

First, we note that without information constraints, it is easy to detect the biased coordinate
with $\Ocal(\log(d)/\rho^2)$ instances. This is formalized in the following theorem,
which is an immediate consequence of Hoeffding's inequality and a union bound:
\begin{theorem}\label{thm:fullupbound}
Consider the hide-and-seek problem defined earlier. Given $mn$ samples, if
$\tilde{J}$ is the coordinate with the highest empirical average, then
\[
{\Pr}_{j}(\tilde{J}=j)\geq 1-2d\exp\left(-\frac{1}{2}mn\rho^2\right).
\]
\end{theorem}

We now show that for this hide-and-seek problem, there is a large regime
where detecting $j$ is information-theoretically possible (by
\thmref{thm:fullupbound}), but any information-constrained protocol
will fail to do so with high probability.

We first show this for $(b,1,m)$ protocols (i.e. protocols which process one
instance at a time, such as bounded-memory online algorithms, and distributed
algorithms where each machine holds a single instance):
\begin{theorem}\label{thm:full1}
Consider the hide-and-seek problem on $d>1$ coordinates, with some bias
$\rho\leq 1/4$ and sample size $m$. Then for any estimate $\tilde{J}$ of the
biased coordinate returned by any $(b,1,m)$ protocol, there exists some
coordinate $j$ such that
\[
{\Pr}_{j}(\tilde{J}=j) \leq \frac{3}{d}+21\sqrt{\frac{m \rho^2 b}{d}}.
\]
\end{theorem}
The theorem implies that any algorithm corresponding to $(b,1,m)$
protocols requires sample size $m\geq \Omega((d/b)/\rho^2)$ to reliably detect some $j$.
When $b$ is polynomially smaller than $d$ (e.g. a constant), we get an
exponential gap compared to constraint-free protocols, which only require
$\Ocal(\log(d)/\rho^2)$ instances. Moreover, \thmref{thm:full1} is optimal up
to log-factors: Consider a $b$-memory online algorithm, which splits the $d$
coordinates into $\Ocal(d/b)$ segments of $\Ocal(b)$ coordinates each, and
sequentially goes over the segments, each time using
$\tilde{\Ocal}(1/\rho^2)$ independent instances to determine if one of the
coordinates in each segment is biased by $\rho$ (assuming $\rho$ is not
exponentially smaller than $b$, this can be done with $\Ocal(b)$ memory by
maintaining the empirical average of each coordinate). This will allow to
detect the biased coordinate, using $\tilde{\Ocal}((d/b)/\rho^2)$ instances.

We now turn to provide an analogous result for general $(b,n,m)$ protocols
(where $n$ is possibly greater than $1$). However, it is a bit weaker in
terms of the dependence on the bias parameter\footnote{The proof of
\thmref{thm:full1} can be applied in the case $n>1$, but the dependence on
$n$ is exponential - see the proof for details.}:
\begin{theorem}\label{thm:full}
Consider the hide-and-seek problem on $d>1$ coordinates, with some bias
$\rho\leq 1/4n$ and sample size $mn$. Then for any estimate $\tilde{J}$ of
the biased coordinate returned by any $(b,n,m)$ protocol, there exists some
coordinate $j$ such that
\[
{\Pr}_{j}(\tilde{J}= j) \leq
\frac{3}{d}+5\sqrt{mn \min\left\{\frac{10 \rho b}{d},\rho^2\right\}}.
\]
\end{theorem}
The theorem implies that any $(b,n,m)$ protocol will require a sample size
$mn$ which is at least\\
$\Omega\left(\max\left\{\frac{(d/b)}{\rho},\frac{1}{\rho^2}\right\}\right)$
in order to detect the biased coordinate. This is larger than the
$\Ocal(\log(d)/\rho^2)$ instances required by constraint-free protocols
whenever $\rho > b\log(d)/d$, and establishes a trade-off between sample
complexity and information complexities such as memory and communication in
this regime.

The proofs of our theorems appear in Appendix \ref{app:proofs}. However, the
technical details may obfuscate the high-level intuition, which we now turn
to explain.

From an information-theoretic viewpoint, our results are based on analyzing
the mutual information between $j$ and $W^t$ in a graphical model as
illustrated in figure \ref{fig:channel}. In this model, the unknown message
$j$ (i.e. the identity of the biased coordinate) is correlated with one of
$d$ independent binary-valued random vectors (one for each coordinate across
the data instances $X^t$). All these random vectors are noisy, and the mutual
information in bits between $X^t_j$ and $j$ can be shown to be on the order
of $n\rho^2$. Without information constraints, it follows that given $m$
instantiations of $X^t$, the total amount of information conveyed on $j$ by
the data is $\Theta(mn\rho^2)$, and if this quantity is larger than
$\log(d)$, then there is enough information to uniquely identify $j$. Note
that no stronger bound can be established with standard statistical
lower-bound techniques, since these do not consider information constraints
internal to the algorithm used.

Indeed, in our information-constrained setting there is an added
complication, since the output $W^t$ can only contain $b$ bits. If $b\ll d$,
then $W^t$ cannot convey all the information on $X^t_1,\ldots,X^t_d$.
Moreover, it will likely convey only little information if it doesn't already
``know'' $j$. For example, $W^t$ may provide a little bit of information on
all $d$ random variables, but then the information conveyed on each (and in
particular, the random variable $X^t_j$ which is correlated with $j$) will be
very small. Alternatively, $W^t$ may provide accurate information on
$\Ocal(b)$ coordinates, but since the relevant random variable $X^t_j$ is not
known, it is likely to `miss' it. The proof therefore relies on the following
components:
\begin{itemize}
    \item No matter what, a $(b,n,m)$ protocol cannot provide more than
        $b/d$ bits of information (in expectation) on $X^t_j$, unless it
        already ``knows'' $j$.
    \item Even if the mutual information between $W^t$ and $X^t_j$ is only
        $b/d$, and the mutual information between $X^t_j$ and $j$ is
        $n\rho^2$, standard information-theoretic tools such as the data
        processing inequality only implies that the mutual information
        between $W^t$ and $j$ is bounded by $\min\{n\rho^2,b/d\}$. We
        essentially prove a stronger information contraction bound, which
        is the \emph{product} of the two terms $\Ocal(\rho^2 b/d)$ when
        $n=1$, and $\Ocal(n\rho b/d)$ for general $n$. At a technical
        level, this is achieved by considering the relative entropy between
        the distributions of $W^t$ with and without a biased coordinate
        $j$, relating it to the $\chi^2$-divergence between these
        distributions (using relatively recent analytic results on
        Csisz\'{a}r f-divergences \cite{Drag00}, \cite{TaKu04}), and
        performing algebraic manipulations to upper bound it by $\rho^2$
        times the mutual information between $W^t$ and $X^t_j$, which is on
        average $b/d$ as discussed earlier. This eventually leads to the
        $m\rho^2 b/d$ term in \thmref{thm:full1}, as well as
        \thmref{thm:full} using somewhat different calculations.
\end{itemize}

\begin{SCfigure}
\centering \caption{Illustration of the relationship between $j$, the
coordinates $1,2,\ldots,j,\ldots,d$ of the sample $X^t$, and the message
$W^t$. The coordinates are independent of each other, and most of them just
output $\pm 1$ uniformly at random. Only $X^t_j$ has a slightly different
distribution and hence contains some information on $j$.}
\includegraphics[scale=1.2]{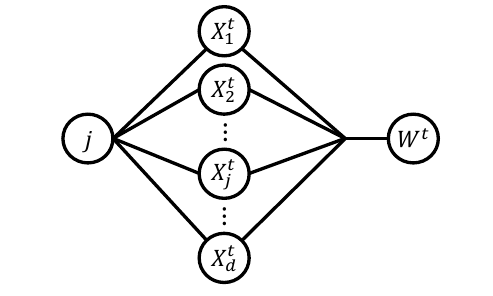}
\label{fig:channel}
\end{SCfigure}

\section{Applications}\label{sec:app}

\subsection{Online Learning with Partial Information}

Consider the standard setting of learning with expert advice, defined as a
game over $T$ rounds, where each round $t$ a loss vector $\ell_t \in [0,1]^d$
is chosen, and the learner (without knowing $\ell_t$) needs to pick an action
$i_t$ from a fixed set $\{1,\ldots,d\}$, after which the learner suffers loss
$\ell_{t,i_t}$. The goal of the learner is to minimize the regret in
hindsight to the any fixed action $i$,
$\sum_{t=1}^{T}\ell_{t,i_t}-\sum_{t=1}^{T}\ell_{t,i}$. We are interested in partial
information variants, where the learner doesn't get to see and use $\ell_t$,
but only some partial information on it. For example, in standard multi-armed
bandits, the learner can only view $\ell_{t,i_t}$.

The following theorem is a corollary of \thmref{thm:full1}, and we provide a
proof in Appendix \ref{app:bandits}.
\begin{theorem}\label{thm:bandits}
    Suppose $d>3$. For any $(b,1,T)$ protocol,
    there is an i.i.d. distribution over loss vectors $\ell_t\in [0,1]^d$ such that for some
    numerical constant $c$,
    \[
    \min_{j}\E\left[\sum_{t=1}^{T}\ell_{t,j_t}-\sum_{t=1}^{T}\ell_{t,j}\right]\geq
    c~\min\left\{T,\sqrt{\frac{d}{b}T}\right\}.
    \]
\end{theorem}
As a result, we get that for any algorithm with any partial information
feedback model (where $b$ bits are extracted from each $d$-dimensional loss
vector), it is impossible to get regret lower than $\Omega(\sqrt{(d/b)T})$
for sufficiently large $T$. Interestingly, this holds even if the algorithm
is allowed to examine each loss vector $\ell_t$ and choose which $b$ bits of
information it wishes to retain. In contrast, full-information algorithms
(e.g. Hedge \cite{freund1997decision}) can get $\Ocal(\sqrt{\log(d)T})$
regret. Without further assumptions on the feedback model, the bound is
optimal up to log-factors, as shown by $\Ocal(\sqrt{(d/b)T})$ upper bounds
for linear or coordinate measurements (where $b$ is the number of
measurements or coordinates seen\footnote{Strictly speaking, if the losses
are continuous-valued, these require arbitrary-precision measurements, but in
any practical implementation we can assume the losses and measurements are
discrete.}) \cite{agarwaloptimal,mannor2011bandits,seldin2014prediction}.
However, the lower bound is more general and applies to any partial feedback
model. For example, we immediately get an $\Omega(\sqrt{(d/k)T}$ regret lower
bound when we are allowed to view $k$ coordinates instead of $1$,
corresponding to (say) the semi-bandit feedback model
(\cite{cesa2012combinatorial}), the side-observation model of
\cite{mannor2011bandits} with a fixed upper bound $k$ on the number of
side-observations. In partial monitoring (\cite{cesa2006prediction}), we get
a $\Omega(d/k)$ lower bound where $k$ is the logarithm of the feedback matrix
width. In learning with partially observed attributes (e.g.
\cite{cesa2011efficient}), a simple reduction implies an
$\Omega(\sqrt{(d/k)T})$ lower bound when we are constrained to view at most
$k$ features of each example.

\subsection{Stochastic Optimization}\label{subsec:stochopt}

We now turn to consider an example from stochastic optimization, where our
goal is to approximately minimize $F(\bh)=\E_Z[f(\bh;Z)]$ given access to $m$
i.i.d. instantiations of $Z$, whose distribution is unknown. This setting has
received much attention in recent years, and can be used to model many
statistical learning problems. In this section, we show a stochastic
optimization problem where information-constrained protocols provably pay a
performance price compared to non-constrained algorithms. We emphasize that
it is going to be a very simple toy problem, and is not meant to represent
anything realistic. We present it for two reasons: First, it illustrates
another type of situation where information-constrained protocols may fail
(in particular, problems involving matrices). Second, the intuition of the
construction is also used in the more realistic problem of sparse PCA and
covariance estimation, considered in the next section.

The construction is as follows: Suppose we wish to solve
$\min_{(\bw,\bv)}F(\bw,\bv) = \E_{Z}[f((\bw,\bv);Z)]$, where
\[
f((\bw,\bv);Z)= \bw^\top Z \bv~~,~~ Z\in [-1,+1]^{d\times d}
\]
and $\bw,\bv$ range over all vectors in the simplex (i.e. $w_i,v_i\geq 0$ and
$\sum_{i=1}^{d}w_i=\sum_{i=1}^{d}v_i=1$). A minimizer of $F(\bw,\bv)$ is $(\be_{i^*},\be_{j^*})$,
where $(i^*,j^*)$ are indices of the matrix entry with minimal mean. Moreover,
by a standard concentration of measure argument, given $m$ i.i.d. instantiations
$Z^1,\ldots,Z^m$ from any distribution over $Z$, then the solution $(\be_{\tilde{I}},\be_{\tilde{J}})$,
where $(\tilde{I},\tilde{J})=\arg\min_{i,j}\frac{1}{m}\sum_{t=1}^{m}Z^t_{i,j}$ are the indices of the
entry with empirically smallest mean, satisfies
$F(\be_{\tilde{I}},\be_{\tilde{J}})\leq \min_{\bw,\bv}F(\bw,\bv)+\Ocal\left(\sqrt{\log(d)/m}\right)$
with high probability.

However, computing $(\tilde{I},\tilde{J})$ as above requires us to track
$d^2$ empirical means, which may be expensive when $d$ is large. If instead
we constrain ourselves to $(b,1,m)$ protocols where $b=\Ocal(d)$ (e.g. any
sort of stochastic gradient method optimization algorithm, whose memory is
linear in the number of parameters), then we claim that we have a lower bound
of $\Omega(\min\{1,\sqrt{d/m}\})$ on the expected error, which is much higher
than the $\Ocal(\sqrt{\log(d)/m})$ upper bound for constraint-free protocols.
This claim is a straightforward consequence of \thmref{thm:full1}: We
consider distributions where $Z\in \{-1,+1\}^{d\times d}$ with probability
$1$, each of the $d^2$ entries is chosen independently, and $\E[Z]$ is zero
except some coordinate $(i^*,j^*)$ where it equals $\Ocal(\sqrt{d/m})$. For
such distributions, getting optimization error smaller than
$\Ocal(\sqrt{d/m})$ reduces to detecting $(i^*,j^*)$, and this in turn
reduces to the hide-and-seek problem defined earlier, over $d^2$ coordinates
and a bias $\rho=\Ocal(\sqrt{d/m})$. However, \thmref{thm:full1} shows that
no $(b,1,m)$ protocol (where $b=\Ocal(d)$) will succeed if $md\rho^2\ll d^2$,
which indeed happens if $\rho$ is small enough.

Similar kind of gaps can be shown using \thmref{thm:full} for general $(b,n,m)$ protocols,
which apply to any special case such as non-interactive distributed learning.

\subsection{Sparse PCA, Sparse Covariance Estimation, and Detecting Correlations}

The sparse PCA problem (\cite{zou2006sparse}) is a standard and well-known
statistical estimation problem, defined as follows: We are given an i.i.d.
sample of vectors $\bx\in \reals^d$, and we assume that there is some
direction, corresponding to some \emph{sparse} vector $\bv$ (of cardinality
at most $k$), such that the variance $\E[(\bv^\top \bx)^2]$ along that
direction is larger than at any other direction. Our goal is to find that
direction.

We will focus here on the simplest possible form of this problem, where the
maximizing direction $\bv$ is assumed to be $2$-sparse, i.e. there are only
$2$ non-zero coordinates $v_i,v_j$. In that case, $\E[(\bv^\top \bx)^2] =
v_1^2\E[x_1^2]+v_2^2\E[x_2^2]+2v_1v_2\E[\bx_i \bx_j]$. Following previous
work (e.g. \cite{BerRig13}), we even assume that $\E[x_i^2]=1$ for all $i$,
in which case the sparse PCA problem reduces to detecting a coordinate pair
$(i^*,j^*)$, $i^*< j^*$ for which $x_{i^*},x_{j^*}$ are maximally correlated.
A special case is a simple and natural sparse covariance estimation problem
(\cite{bien2011sparse,cai2011adaptive}), where we assume that all covariates
are uncorrelated ($\E[x_i x_j]=0$) except for a unique correlated pair of
covariates $(i^*,j^*)$ which we need to detect.

This setting bears a resemblance to the example seen in the context of
stochastic optimization in section \ref{subsec:stochopt}: We have a $d\times
d$ stochastic matrix $\bx \bx^\top$, and we need to detect an off-diagonal
biased entry at location $(i^*,j^*)$. Unfortunately, these stochastic
matrices are rank-1, and do not have independent entries as in the example
considered in section \ref{subsec:stochopt}. Instead, we use a more delicate
construction, relying on distributions supported on sparse vectors. The
intuition is that then each instantiation of $\bx\bx^\top$ is sparse, and the
situation can be reduced to a variant of our hide-and-seek problem where only
a few coordinates are non-zero at a time. The theorem below establishes
performance gaps between constraint-free protocols (in particular, a simple
plug-in estimator), and any $(b,n,m)$ protocol for a specific choice of $n$,
or any $b$-memory online protocol (See \secref{sec:prelim}).

\begin{theorem}\label{thm:sparsePCA}
Consider the class of $2$-sparse PCA (or covariance estimation) problems in
$d\geq 9$ dimensions as described above, and all distributions such that:
\begin{enumerate}
  \item $\E[x_i^2]=1$ for all $i$.
  \item For a unique pair of distinct coordinates $(i^{*},j^{*})$, it holds
      that $\E[x_{i^{*}} x_{j^{*}}]=\tau>0$, whereas $\E[x_{i} x_{j}]=0$
      for all distinct coordinate pairs $(i,j)\neq (i^*,j^*)$.
  \item For any $i<j$, if $\widetilde{x_i x_j}$ is the empirical average of
      $x_i x_j$ over $m$ i.i.d. instances, then\\ $\Pr\left(|\widetilde{x_i
      x_j}-\E[x_i x_j]|\geq \frac{\tau}{2}\right) \leq
      2\exp\left(-m\tau^2/6\right)$.
\end{enumerate}
Then the following holds:
\begin{itemize}
    \item Let $(\tilde{I},\tilde{J})=\arg\max_{i<j}\widetilde{x_i x_j}$.
        Then for any distribution as above,
        $\Pr((\tilde{I},\tilde{J})=(i^*,j^*))\geq 1-d^2\exp(-m\tau^2/6)$.
        In particular, when the bias $\tau$ equals $\Theta(1/d\log(d))$,
        \[
        \Pr((\tilde{I},\tilde{J})=(i^*,j^*))\geq 1-d^2\exp\left(-\Omega\left(\frac{m}{d^2\log^2(d)}\right)\right). \]
    \item For any estimate $(\tilde{I},\tilde{J})$ of $(i^*,j^*)$ returned
        by any $b$-memory online protocol using $m$ instances, or any
        $(b,d(d-1),\lfloor \frac{m}{d(d-1)}\rfloor)$ protocol, there exists
        a distribution with bias $\tau=\Theta(1/d\log(d))$ as above such
        that
            \[
            \Pr\left((\tilde{I},\tilde{J})=(i^*,j^*)\right)\leq \Ocal\left(\frac{1}{d^2}+\sqrt{\frac{m}{d^4/b}}\right).
            \]
\end{itemize}
\end{theorem}
The theorem implies that in the regime where $b\ll d^2/\log^2(d)$, we can
choose any $m$ such that $\frac{d^4}{b}\gg m \gg d^2\log^2(d)$, and get that
the chances of the protocol detecting $(i^*,j^*)$ are arbitrarily small, even
though the empirical average reveals $(i^*,j^*)$ with arbitrarily high
probability. Thus, in this sparse PCA / covariance estimation setting, any
online algorithm with sub-quadratic memory cannot be statistically optimal
for all sample sizes. The same holds for any $(b,n,m)$ protocol in an
appropriate regime of $(n,m)$, such as distributed algorithms as discussed
earlier.

To the best of our knowledge, this is the first result which explicitly shows
that \emph{memory} constraints can incur a statistical cost for a standard
estimation problem. It is interesting that sparse PCA was also shown recently
to be affected by \emph{computational} constraints on the algorithm's runtime
(\cite{BerRig13}).

The proof appears in Appendix \ref{app:sparsePCA}. Besides using a somewhat
different hide-and-seek construction as mentioned earlier, it also relies on
the simple but powerful observation that any $b$-memory online protocol is
also a $(b,\kappa,\lfloor m/\kappa \rfloor)$ protocol for arbitrary $\kappa$.
Therefore, we only need to prove the theorem for $(b,\kappa,\lfloor m/\kappa
\rfloor)$ for some $\kappa$ (chosen to equal $d(d-1)$ in our case) to automatically get the
same result for $b$-memory protocols.

The theorem shows a performance gap in the regime where $\frac{d^4}{b}\gg m
\gg d^2\log^2(d)$. However, it is possible to establish such gaps for similar
problems already when $m$ is linear in $b$ (up to log-factors) - see the
proof in Appendix \ref{app:sparsePCA} for more details. Also, the
distribution on which information-constrained protocols are shown to fail
satisfies the theorem's conditions, but is `spiky' and rather unnatural.
Proving a similar result for a `natural' data distribution (e.g. Gaussian)
remains an interesting open problem.

\section{Discussion and Open Questions}\label{sec:discussion}

In this paper, we investigated cases where a generic type of
information-constrained algorithm has strictly inferior statistical
performance compared to constraint-free algorithms. As special cases, we
demonstrated such gaps for memory-constrained and communication-constrained
algorithms (e.g. in the context of sparse PCA and covariance estimation), as
well as online learning with partial information and stochastic optimization.
These results are based on explicitly considering the information-theoretic
structure of the problem, and depend only on the number of bits extracted
from each data batch. We believe these results form a first step in a deeper
understanding of how information constraints affect learning ability.

Several questions remain open. One question is whether \thmref{thm:full} can
be improved. We conjecture this is true, and that the bound should actually
depend on $mn\rho^2 b/d$ rather than $mn\min\{\rho b/d,\rho^2\}$, and hold
for any $\rho\leq \Ocal(1/\sqrt{n})$. This would allow us, for instance, to show the same
type of performance gaps for $(b,1,m)$ protocols and $(b,n,m)$ protocols.

A second open question is whether there are convex stochastic optimization
problems, for which online or distributed algorithms are provably inferior to
constraint-free algorithms (the example discussed in section
\ref{subsec:stochopt} refers to an easily-solvable yet non-convex problem).
Due to the large current effort in developing scalable algorithms for convex
learning problems, this would establish that one must pay a statistical price
for using such memory-and-time efficient algorithms.

A third open question is whether the results for non-interactive (or serial)
distributed algorithms can be extended to more interactive algorithms, where
the different machines can communicate over several rounds. There is a rich
literature on the communication complexity of interactive distributed
algorithms within theoretical computer science, but it is not clear how to
`import' these results to a statistical setting based on i.i.d. data.

A fourth open question relates to the sparse-PCA / covariance estimation
result. The hardness result we gave for information-constrained protocols
uses a tailored distribution, which has a sufficiently controlled tail
behavior but is `spiky' and not sub-Gaussian uniformly in the dimension.
Thus, it would be interesting to establish similar hardness results for
`natural' distributions (e.g. Gaussian).

More generally, there is much work remaining in extending the results here to
other learning problems and other information constraints.

\subsubsection*{Acknowledgments}
This research is supported by the Intel ICRI-CI Institute, Israel Science
Foundation grant 425/13, and an FP7 Marie Curie CIG grant. We thank John
Duchi, Yevgeny Seldin and Yuchen Zhang for helpful comments.

\bibliographystyle{plain}
\bibliography{bib}

\appendix

\section{Proofs}\label{app:proofs}

The proofs use several standard quantities and results from information
theory -- see Appendix \ref{app:info} for more details. They also make use of
a several auxiliary lemmas (presented in \subsecref{subsec:auxiliary}),
including a simple but key lemma (\lemref{lem:infoexp}) which quantifies how
information-constrained protocols cannot provide information on all
coordinates simultaneously.

\subsection{Auxiliary Lemmas}\label{subsec:auxiliary}

\begin{lemma}\label{lem:totvar}
Suppose that $d>1$, and for some fixed distribution ${\Pr}_{0}(\cdot)$ over
the messages $w^1,\ldots,w^m$ computed by an information-constrained
protocol, it holds that
\[
\sqrt{\frac{2}{d}\sum_{j=1}^{d}D_{kl}\left({\Pr}_{0}(w^1\ldots w^m)\middle|\middle|{\Pr}_{j}(w^1\ldots w^m)\right)} \leq B.
\]
Then there exist some $j$ such that
\[
\Pr(\tilde{J}=j) \leq \frac{3}{d}+2B.
\]
\end{lemma}
\begin{proof}
By concavity of the square root, we have
\[
\sqrt{\frac{2}{d}\sum_{j=1}^{d}D_{kl}\left({\Pr}_{0}(w^1\ldots w^m)\middle|\middle|{\Pr}_{j}(w^1\ldots w^m)\right)}
\geq \frac{1}{d}\sum_{j=1}^{d}\sqrt{2~D_{kl}\left({\Pr}_{0}(w^1\ldots w^m)\middle|\middle|{\Pr}_{j}(w^1\ldots w^m)\right)}.
\]
Using Pinsker's inequality and the fact that $\tilde{J}$ is some function of
the messages $w^1,\ldots,w^m$ (independent of the data distribution), this is
at least
\begin{align*}
&\frac{1}{d}\sum_{j=1}^{d}\sum_{w^1\ldots w^m}\left|{\Pr}_{0}(w^1\ldots w^m)-{\Pr}_{j}(w^1\ldots w^m)\right| \\
&\geq \frac{1}{d}\sum_{j=1}^{d}\left|\sum_{w^1\ldots w^m}\left({\Pr}_{0}(w^1\ldots w^m)-{\Pr}_{j}(w^1\ldots w^m)\right)\Pr\left(\tilde{J}|w^1\ldots w^m\right)\right|\\
&\geq \frac{1}{d}\sum_{j=1}^{d}|{\Pr}_{0}(\tilde{J}=j)-{\Pr}_{j}(\tilde{J}=j)|.
\end{align*}
Thus, we may assume that
\[
\frac{1}{d}\sum_{j=1}^{d}|{\Pr}_{0}(\tilde{J}=j)-{\Pr}_{j}(\tilde{J}=j)| \leq B.
\]
The argument now uses a basic variant of the probabilistic method. If the
expression above is at most $B$, then for at least $d/2$ values of $j$, it
holds that $|{\Pr}_{0}(\tilde{J=j})-{\Pr}_{j}(\tilde{J}=j)|\leq 2B$. Also,
since $\sum_{j=1}^{d}{\Pr}_{0}(\tilde{J}=j) = 1$, then for at least $2d/3$
values of $j$, it holds that ${\Pr}_{0}(\tilde{J}=j) \leq 3/d$. Combining the
two observations, and assuming that $d>1$, it means there must exist some
value of $j$ such that $|{\Pr}_{0}(\tilde{J})-{\Pr}_{j}(\tilde{J}=j)|\leq
2B$, \emph{as well as} ${\Pr}_{0}(\tilde{J}=j)\leq 3/d$, hence
${\Pr}_{j}(\tilde{J}=j)\leq \frac{3}{d}+2B$ as required.
\end{proof}

\begin{lemma}\label{lem:decomposition}
Let $p,q$ be distributions over a product domain $A_1\times A_2\times\ldots
\times A_d$, where each $A_i$ is a finite set. Suppose that for some $j\in
\{1,\ldots,d\}$, the following inequality holds for all
$\bz=(z_1,\ldots,z_d)\in A_1\times\ldots \times A_d$:
\[
  p(\{z_i\}_{i\neq j}|z_j)=q(\{z_i\}_{i\neq j}|z_j).
\]
Also, let $E$ be an event such that $p(E|\bz)=q(E|\bz)$ for all $\bz$. Then
\[
p(E) = \sum_{z_j}p(z_j)q(E|z_j).
\]
\end{lemma}
\begin{proof}
\begin{align*}
  p(E) &= \sum_{\bz}p(\bz)p(E|\bz) = \sum_{\bz}p(\bz)q(E|\bz)\\
  &= \sum_{z_j}p(z_j)\sum_{\{z_i\}_{i\neq j}}p(\{z_j\}_{i\neq j}|z_j)q(E|z_j,\{z_i\}_{i\neq j})\\
  &= \sum_{z_j}p(z_j)\sum_{\{z_i\}_{i\neq j}}q(\{z_j\}_{i\neq j}|z_j)q(E|z_j,\{z_i\}_{i\neq j})\\
  &= \sum_{z_j}p(z_j)q(E|z_j).
\end{align*}
\end{proof}

\begin{lemma}[\cite{Drag00}, Proposition 1] \label{lem:klreverse}
Let $p,q$ be two distributions on a discrete set, such that $\max_x
\frac{p(x)}{q(x)}\leq c$. Then
\[
D_{kl}\left(p(\cdot)||q(\cdot)\right) \leq c~ D_{kl}\left(q(\cdot)||p(\cdot)\right).
\]
\end{lemma}

\begin{lemma}[\cite{Drag00}, Proposition 2 and Remark 4] \label{lem:klchi}
Let $p,q$ be two distributions on a discrete set, such that $\max_x
\frac{p(x)}{q(x)}\leq c$. Also, let $D_{\chi^2}(p(\cdot)||q(\cdot))=\sum_{x}\frac{(p(x)-q(x))^2}{q(x)}$
denote the $\chi^2$-divergence between the distributions $p,q$. Then
\[
D_{kl}\left(p(\cdot)||q(\cdot)\right) \leq D_{\chi^2}\left(p(\cdot)||q(\cdot)\right)
\leq 2c~ D_{kl}\left(p(\cdot)||q(\cdot)\right).
\]
\end{lemma}

\begin{lemma}\label{lem:ballsbins}
Suppose we throw $n$ balls independently and uniformly at random into $d>1$
bins, and let $K_1,\ldots K_d$ denote the number of balls in each of the $d$
bins. Then for any $\epsilon\geq 0$ such that $\epsilon \leq
\min\{\frac{1}{6},\frac{1}{2\log(d)},\frac{d}{3n}\}$, it holds that
\[
\E\left[\exp\left(\epsilon \max_j K_j\right)\right] < 13.
\]
\end{lemma}
\begin{proof}
Each $K_j$ can be written as $\sum_{i=1}^{n}\mathbf{1}(\text{ball $i$ fell
into bin $j$})$, and has expectation $n/d$. Therefore, by a standard
multiplicative Chernoff bound, for any $\gamma\geq 0$,
\[
\Pr\left(K_j> (1+\gamma)\frac{n}{d}\right)\leq \exp\left(-\frac{\gamma^2}{2(1+\gamma)}\frac{n}{d}\right).
\]
By a union bound, this implies that
\[
\Pr\left(\max_j K_j> (1+\gamma)\frac{n}{d}\right)\leq \sum_{j=1}^{d}
\Pr\left(K_j > (1+\gamma)\frac{n}{d}\right) \leq d\exp\left(-\frac{\gamma^2}{2(1+\gamma)}\frac{n}{d}\right).
\]
In particular, if $\gamma+1 \geq 6$, we can upper bound the above by the
simpler expression $\exp(-(1+\gamma)n/3d)$. Letting $\tau=\gamma+1$, we get
that for any $\tau\geq 6$,
\begin{equation}\label{eq:6bound}
\Pr\left(\max_j K_j > \tau \frac{n}{d}\right)\leq d\exp\left(-\frac{\tau n}{3d}\right).
\end{equation}
Define $c=\max\{8,d^{3\epsilon}\}$. Using the inequality above and the
non-negativity of $\exp(\epsilon \max_j K_j)$, we have
\begin{align*}
\E\left[\exp(\epsilon \max_j K_j)\right] &= \int_{t=0}^{\infty}\Pr\left(\exp(\epsilon \max_j K_j)\geq t\right)dt\\
&\leq c + \int_{t=c}^{\infty}\Pr\left(\exp(\epsilon \max_j K_j)\geq t\right)dt\\
&= c + \int_{t=c}^{\infty}\Pr\left(\max_j K_j \geq \frac{\log(t)}{\epsilon}\right)dt\\
&= c+ \int_{t=c}^{\infty}\Pr\left(\max_j K_j \geq \frac{\log(t)d}{\epsilon n}\frac{n}{d}\right)dt
\end{align*}
Since we assume $\epsilon\leq d/3n$ and $c\geq 8$, it holds that
$\exp(6\epsilon n/d)\leq \exp(2) < 8 \leq c$, which implies
$\log(c)d/\epsilon n \geq 6$. Therefore, for any $t\geq c$, it holds that
$\log(t)d/\epsilon n \geq 6$. This allows us to use \eqref{eq:6bound} to
upper bound the expression above by
\[
c+ d\int_{t=c}^{\infty}\exp\left(-\frac{\log(t)d}{3\epsilon n}\frac{n}{d}\right)dt
~=~
c+ d\int_{t=c}^{\infty}t^{-1/3\epsilon}dt.
\]
Since we assume $\epsilon \leq 1/6$, we have $1/(3\epsilon) \geq 2$, and
therefore we can solve the integration to get
\[
c+\frac{d}{\frac{1}{3\epsilon}-1}c^{1-\frac{1}{3\epsilon}}\leq  c+dc^{1-\frac{1}{3\epsilon}}.
\]
Using the value of $c$, and since $1-\frac{1}{3\epsilon}\leq -1$, this is at
most
\[
\max\{8,d^{3\epsilon}\}+d*\left(d^{3\epsilon}\right)^{1-\frac{1}{3\epsilon}}
~=~ \max\{8,d^{3\epsilon}\}+d^{3\epsilon}.
\]
Since $\epsilon \leq 1/2\log(d)$, this is at most
\[
\max\{8,\exp(3/2)\}+\exp(3/2)< 13
\]
as required.
\end{proof}

\begin{lemma}\label{lem:infoexp}
Let $Z_1,\ldots,Z_d$ be independent random variables, and let $W$ be a random
variable which can take at most $2^b$ values. Then
\[
\frac{1}{d}\sum_{j=1}^{d} I(W;Z_j) \leq \frac{b}{d}.
\]
\end{lemma}
\begin{proof}
We have
\[
\frac{1}{d}\sum_{j=1}^{d}I(W;Z_j) =
\frac{1}{d}\sum_{j=1}^{d}\left(H(Z_j)-H(Z_j|W)\right).
\]
Using the fact that $\sum_{j=1}^{d}{H}(Z_j|W) \geq H(Z_1\ldots,Z_d|W) $, this
is at most
\begin{align}
&\frac{1}{d}\sum_{j=1}^{d}H(Z_j)-\frac{1}{d}H(Z_1\ldots Z_d|W)\notag\\
&=~\frac{1}{d}\sum_{j=1}^{d}H(Z_j)
-\frac{1}{d}\left(H(Z_1\ldots Z_d)-I(Z_1\ldots Z_d;W)\right)\notag\\
&= \frac{1}{d}I(Z_1\ldots Z_d;W)+
\frac{1}{d}\left(\sum_{j=1}^{d}H(Z_j)
-H(Z_1\ldots Z_d)\right).\label{eq:diddid}
\end{align}
Since $Z_1\ldots Z_d$ are independent, $\sum_{j=1}^{d}H(Z_j)=H(Z_1\ldots
Z_d)$, hence the above equals
\[
\frac{1}{d}I(Z_1\ldots Z_d;W)=\frac{1}{d}\left(H(W)-H(W|Z_1\ldots Z_d)\right)\leq \frac{1}{d}H(W),
\]
which is at most $b/d$ since $W$ is only allowed to have $2^b$ values.
\end{proof}

\subsection{Proof of \thmref{thm:full1}}\label{subsec:thmfull1proof}

We will actually prove a more general result, stating that for any $(b,n,m)$ protocol,
\[
{\Pr}_{j}(\tilde{J}=j) \leq \frac{3}{d}+14.3\sqrt{mn 2^n \frac{\rho^2b}{d}}.
\]
The result stated in the theorem follows in the case $n=1$.

The proof builds on the auxiliary lemmas presented in Appendix
\ref{subsec:auxiliary}.

On top of the distributions ${\Pr}_{j}(\cdot)$ defined in the hide-and-seek
problem (Definition \ref{def:prob1}), we define an additional `reference'
distribution ${\Pr}_{0}(\cdot)$, which corresponds to the instances $\bx$
chosen uniformly at random from $\{-1,+1\}^d$ (i.e. there is no biased
coordinate).

Let $w^1,\ldots,w^m$ denote the messages computed by the protocol. It is enough to prove that
\begin{equation}\label{eq:11ppr0}
\frac{2}{d}\sum_{j=1}^{d}D_{kl}\left({\Pr}_{0}(w^1\ldots w^m)\middle|\middle|{\Pr}_{j}(w^1\ldots w^m)\right) \leq 51 mn2^n \rho^2 b/d,
\end{equation}
since then by applying \lemref{lem:totvar}, we get that for some $j$,
${\Pr}_{j}(\tilde{J}=j) \leq (3/d)+2\sqrt{51 mn2^n\rho^2 b/d} \leq
(3/d)+14.3\sqrt{mn2^n \rho^2 b/d}$ as required.

Using the chain rule, the left hand side in \eqref{eq:11ppr0} equals
\begin{align}
&\frac{2}{d}\sum_{j=1}^{d}\sum_{t=1}^{m}\E_{w^1\ldots w^{t-1}\sim {\Pr}_{0}}\left[D_{kl}\left({\Pr}_{0}(w^t|w^1\ldots w^{t-1})||{\Pr}_{j}(w^t|w^1\ldots w^{t-1})\right)\right]\notag\\
&=2\sum_{t=1}^{m}\E_{w^1\ldots w^{t-1}\sim {\Pr}_{0}}\left[\frac{1}{d}\sum_{j=1}^{d}D_{kl}\left({\Pr}_{0}(w^t|w^1\ldots w^{t-1})||{\Pr}_{j}(w^t|w^1\ldots w^{t-1})\right)\right]
\label{eq:11whattobound}
\end{align}
Let us focus on a particular choice of $t$ and values $w^1\ldots w^{t-1}$. To
simplify the presentation, we drop the $^t$ superscript from the message
$w^t$, and denote the previous messages $w^1\ldots w^{t-1}$ as $\hat{w}$.
Thus, we consider the quantity
\begin{equation}\label{eq:11i0}
\frac{1}{d}\sum_{j=1}^{d}D_{kl}\left({\Pr}_{0}(w|\hat{w})||{\Pr}_{j}(w|\hat{w})\right).
\end{equation}
Recall that $w$ is some function of $\hat{w}$ and a set of $n$ independent
instances received in the current round. Let $\bx_j$ denote the vector of
values at coordinate $j$ across these $n$ instances. Clearly, under $\Pr_j$,
every $\bx_i$ for $i\neq j$ is uniformly distributed on $\{-1,+1\}^n$,
whereas each entry of $\bx_j$ equals $1$ with probability $\frac{1}{2}+\rho$,
and $-1$ otherwise.

First, we argue that by \lemref{lem:decomposition}, for any $w,\hat{w}$, we have
\begin{equation}\label{eq:11decomp}
{\Pr}_{j}(w|\hat{w}) ={\Pr}_{0}(w|\hat{w}) \sum_{\bx_j}{\Pr}_{j}(\bx_j|\hat{\bw}) =
\sum_{\bx_j}{\Pr}_{0}(w|\hat{w}){\Pr}_{j}(\bx_j|\hat{\bw})
= \sum_{\bx_j}{\Pr}_{0}(w|\hat{w}){\Pr}_{j}(\bx_j).
\end{equation}
This follows by applying the lemma on
$p(\cdot)={\Pr}_{j}(\cdot|\hat{w})$,$q(\cdot)={\Pr}_{0}(\cdot|\hat{w})$
and $A_i = \{-1,+1\}^n$ (i.e. the vector of values at a single coordinate $i$
across the $n$ data points), and noting the $\bx_j$ is independent of $\hat{w}$.
The lemma's conditions are satisfied since
$\bx_i$ for $i\neq j$ has the same distribution under
${\Pr}_{0}(\cdot|\hat{w})$ and ${\Pr}_{j}(\cdot|\hat{w})$, and also $w$ is
only a function of $\bx_1\ldots \bx_d$ and $\hat{w}$.

Using \lemref{lem:klreverse} and \lemref{lem:klchi}, we have the following.
\begin{align}
  D_{kl}\left({\Pr}_{0}(w|\hat{w})||{\Pr}_{j}(w|\hat{w})\right) &\leq
  \max_{w}\left(\frac{{\Pr}_{0}(w|\hat{w})}{{\Pr}_{j}(w|\hat{w})}\right)
  D_{kl}\left({\Pr}_{j}(w|\hat{w})||{\Pr}_{0}(w|\hat{w})\right)\notag\\
&\leq
  \max_{w}\left(\frac{{\Pr}_{0}(w|\hat{w})}{{\Pr}_{j}(w|\hat{w})}\right)
  D_{\chi^2}\left({\Pr}_{j}(w|\hat{w})||{\Pr}_{0}(w|\hat{w})\right)\notag\\
  &= \max_{w}\left(\frac{{\Pr}_{0}(w|\hat{w})}{{\Pr}_{j}(w|\hat{w})}\right)
  \sum_w\frac{\left({\Pr}_{j}(w|\hat{w})-{\Pr}_{0}(w|\hat{w})\right)^2}{{\Pr}_{0}(w|\hat{w})}\label{eq:11dmaxd}
\end{align}

Let us consider the max term and the sum seperately. Using \eqref{eq:11decomp} and the fact that $\rho\leq 1/4n$, we have
\begin{align}
\max_{w}\left(\frac{{\Pr}_{0}(w|\hat{w})}{{\Pr}_{j}(w|\hat{w})}\right)
  &= \max_{w}\left(\frac{\sum_{\bx_j}{\Pr}_{0}(w|\bx_j,\hat{w}){\Pr}_{0}(\bx_j)}
  {\sum_{\bx_j}{\Pr}_{0}(w|\bx_j,\hat{w}){\Pr}_{j}(\bx_j)}\right)\notag\\
  &\leq \max_{\bx_j}\left(\frac{{\Pr}_{0}(\bx_j)}{{\Pr}_{j}(\bx_j)}\right)\notag\\
  &= \left(\frac{1/2}{1/2-\rho}\right)^n
  \leq (1+4\rho)^n \leq (1+1/n)^n \leq \exp(1).\label{eq:11dmax}
\end{align}

As to the sum term in \eqref{eq:11dmaxd}, using \eqref{eq:11decomp} and the Cauchy-Schwartz inequality, we have
\begin{align}
  \sum_w\frac{\left({\Pr}_{j}(w|\hat{w})-{\Pr}_{0}(w|\hat{w})\right)^2}{{\Pr}_{0}(w|\hat{w})}
  &=\sum_w\frac{\left(\sum_{\bx_j}{\Pr}_{0}(w|\bx_j,\hat{w})\left({\Pr}_{j}(\bx_j)-{\Pr}_{0}(\bx_j)\right)\right)^2}{{\Pr}_{0}(w|\hat{w})}\notag\\ &=\sum_w\frac{\left(\sum_{\bx_j}\left({\Pr}_{0}(w|\bx_j,\hat{w})-{\Pr}_{0}(w|\hat{w})\right)\left({\Pr}_{j}(\bx_j)-{\Pr}_{0}(\bx_j)\right)\right)^2}{{\Pr}_{0}(w|\hat{w})}\notag\\
  &\leq\sum_w\frac{\sum_{\bx_j}\left({\Pr}_{0}(w|\bx_j,\hat{w})-{\Pr}_{0}(w|\hat{w})\right)^2\sum_{\bx_j}\left({\Pr}_{j}(\bx_j)-{\Pr}_{0}(\bx_j)\right)^2}{{\Pr}_{0}(w|\hat{w})}\notag\\
  &=\sum_{\bx_j}\left({\Pr}_{j}(\bx_j)-{\Pr}_{0}(\bx_j)\right)^2\sum_{\bx_j}
  D_{\chi^2}\left({\Pr}_{0}(w|\bx_j,\hat{w})||{\Pr}_{0}(w|\hat{w})\right)\label{eq:11dd}.
\end{align}
where we used the definition of $\chi^2$-divergence as specified in \lemref{lem:klchi}.
Again, we will consider each sum separately. Applying \lemref{lem:klchi}
and \eqref{eq:11decomp}, we have
\begin{align}
D_{\chi^2}\left({\Pr}_{0}(w|\bx_j,\hat{w})||{\Pr}_{0}(w|\hat{w})\right)
&\leq 2\max_w\left(\frac{{\Pr}_{0}(w|\bx_j,\hat{w})}{{\Pr}_{0}(w|\hat{w})}\right)
D_{kl}\left({\Pr}_{0}(w|\bx_j,\hat{w})||{\Pr}_{0}(w|\hat{w})\right)\notag\\
&=2\max_w\left(\frac{{\Pr}_{0}(w|\bx_j,\hat{w})}{\sum_{\bx_j}{\Pr}_{0}(w|\bx_j,\hat{w}){\Pr}_{0}(\bx_j)}\right)
D_{kl}\left({\Pr}_{0}(w|\bx_j,\hat{w})||{\Pr}_{0}(w|\hat{w})\right)\notag\\
&=2\max_w\left(\frac{{\Pr}_{0}(w|\bx_j,\hat{w})}{\frac{1}{2^n}\sum_{\bx_j}{\Pr}_{0}(w|\bx_j,\hat{w})}\right)
D_{kl}\left({\Pr}_{0}(w|\bx_j,\hat{w})||{\Pr}_{0}(w|\hat{w})\right)\notag\\
&\leq 2^{n+1}D_{kl}\left({\Pr}_{0}(w|\bx_j,\hat{w})||{\Pr}_{0}(w|\hat{w})\right)\label{eq:11dd1}
\end{align}
Moreover, by definition of ${\Pr}_{0}$ and ${\Pr}_{j}$, and using the fact that each coordinate of $\bx_j$
takes values in $\{-1,+1\}$, we have
\begin{align}
&\sum_{\bx_j}\left({\Pr}_{j}(\bx_j)-{\Pr}_{0}(\bx_j)\right)^2 = \sum_{\bx_j}\left(\prod_{i=1}^{n}\left(\frac{1}{2}+\rho x_{j,i}\right)-\frac{1}{2^n}\right)^2\notag\\
&= \frac{1}{4^n}\sum_{\bx_j}\left(\prod_{i=1}^{n}\left(1+2\rho x_{j,i}\right)-1\right)^2
= \frac{1}{4^n}\sum_{\bx_j}\left(\prod_{i=1}^{n}\left(1+2\rho x_{j,i}\right)^2-2
\prod_{i=1}^{n}\left(1+2\rho x_{j,i}\right)+1\right)\notag\\
&= \frac{1}{4^n}\left(\prod_{i=1}^{n}\sum_{x_{j,i}}\left(1+2\rho x_{j,i}\right)^2
-2\prod_{i=1}^{n}\sum_{x_{j,i}}\left(1+2\rho x_{j,i}\right)+2^n\right)\notag\\
&= \frac{1}{4^n}\left((2+8\rho^2)^n-2^{n+1}+2^n\right)
= \frac{1}{2^n}\left((1+4\rho^2)^n-1\right)\notag\\
&= \frac{1}{2^n}\left(\left(1+\frac{4n\rho^2}{n}\right)^n-1\right)
\leq \frac{1}{2^n}\left(\exp(4n\rho^2)-1\right)\leq \frac{4.6}{2^n}n\rho^2, \label{eq:11dd2}
\end{align}
where in the last inequality we used the fact that $4n\rho^2\leq
4n(1/4n)^2\leq 0.25$, and $\exp(x)\leq 1+1.14 x$ for any $x\in [0,0.25]$.
Plugging in \eqref{eq:11dd1} and \eqref{eq:11dd2} back into \eqref{eq:11dd},
we get that
\[
\sum_w\frac{\left({\Pr}_{j}(w|\hat{w})-{\Pr}_{0}(w|\hat{w})\right)^2}{{\Pr}_{0}(w|\hat{w})}\leq
9.2 n\rho^2~\sum_{\bx_j}D_{kl}\left({\Pr}_{0}(w|\bx_j,\hat{w})||{\Pr}_{0}(w|\hat{w})\right).
\]
Plugging this in turn, together with \eqref{eq:11dmax}, into \eqref{eq:11dmaxd}, we get overall that
\[
D_{kl}\left({\Pr}_{0}(w|\hat{w})||{\Pr}_{j}(w|\hat{w})\right)\leq
9.2\exp(1) n\rho^2~\sum_{\bx_j}D_{kl}\left({\Pr}_{0}(w|\bx_j,\hat{w})||{\Pr}_{0}(w|\hat{w})\right).
\]
This expression can be equivalently written as
\begin{align*}
9.2&\exp(1)n2^n\rho^2~\sum_{\bx_j}\frac{1}{2^n} D_{kl}\left({\Pr}_{0}(w|\bx_j,\hat{w})||{\Pr}_{0}(w|\hat{w})\right)\\
&= 9.2\exp(1)n2^n \rho^2~\sum_{\bx_j}{\Pr}_{0}(\bx_j|\hat{w})D_{kl}\left({\Pr}_{0}(w|\bx_j,\hat{w})||{\Pr}_{0}(w|\hat{w})\right)\\
&= 9.2\exp(1)n 2^n \rho^2 I_{{\Pr}_{0}(\cdot|\hat{w})}(w;\bx_j)\\
\end{align*}
where $I_{{\Pr}_{0}(\cdot|\hat{w})}(w;\bx_j)$ denotes the mutual information
between $w$ and $\bx_j$, under the (uniform) distribution on $\bx_j$ induced
by ${\Pr}_{0}(\cdot|\hat{w})$. This allows us to upper bound \eqref{eq:11i0}
as follows:
\[
\frac{1}{d}\sum_{j=1}^{d}D_{kl}\left({\Pr}_{0}(w|\hat{w})||{\Pr}_{j}(w|\hat{w})\right)
~\leq~
9.2\exp(1)n 2^n \rho^2 \frac{1}{d}\sum_{j=1}^{d}I_{{\Pr}_{0}(\cdot|\hat{w})}(w;\bx_j).
\]
Since $\bx_1,\ldots,\bx_d$ are independent of each other and $w$ contains at most $b$ bits, we can use the key  \lemref{lem:infoexp} to upper bound the above by $9.2\exp(1)n 2^n \rho^2 b/d$.

To summarize, this expression constitutes an upper bound on \eqref{eq:11i0},
i.e. on any individual term inside the expectation in
\eqref{eq:11whattobound}. Thus, we can upper bound \eqref{eq:11whattobound}
by $18.4\exp(1) m n 2^n\rho^2 b/d ~<~  51 mn 2^n \rho^2 b/d$. This shows that
\eqref{eq:11ppr0} indeed holds, which as explained earlier implies the
required result.

\subsection{Proof of \thmref{thm:full}}\label{subsec:thmfullproof}

The proof builds on the auxiliary lemmas presented in Appendix
\ref{subsec:auxiliary}. It begins similarly to the proof of
\thmref{thm:full1}, but soon diverges.

On top of the distributions ${\Pr}_{j}(\cdot)$ defined in the hide-and-seek
problem (Definition \ref{def:prob1}), we define an additional `reference'
distribution ${\Pr}_{0}(\cdot)$, which corresponds to the instances $\bx$
chosen uniformly at random from $\{-1,+1\}^d$ (i.e. there is no biased
coordinate).

Let $w^1,\ldots,w^m$ denote the messages computed by the protocol. To show
the upper bound, it is enough to prove that
\begin{equation}\label{eq:ppr0}
\frac{2}{d}\sum_{j=1}^{d}D_{kl}\left({\Pr}_{0}(w^1\ldots w^m)\middle|\middle|{\Pr}_{j}(w^1\ldots w^m)\right) \leq \min\left\{60\frac{mn\rho b}{d},6mn\rho^2\right\}
\end{equation}
since then by applying \lemref{lem:totvar}, we get that for some $j$,
${\Pr}_{j}(\tilde{J}=j) \leq (3/d)+2\sqrt{\min\{60 mn\rho b/d,6mn\rho^2\}} \leq
(3/d)+5\sqrt{mn \min\{10\rho b/d,\rho^2\}}$ as required.

Using the chain rule, the left hand side in \eqref{eq:ppr0} equals
\begin{align}
&\frac{2}{d}\sum_{j=1}^{d}\sum_{t=1}^{m}\E_{w^1\ldots w^{t-1}\sim {\Pr}_{0}}\left[D_{kl}\left({\Pr}_{0}(w^t|w^1\ldots w^{t-1})||{\Pr}_{j}(w^t|w^1\ldots w^{t-1})\right)\right]\notag\\
&=2\sum_{t=1}^{m}\E_{w^1\ldots w^{t-1}\sim {\Pr}_{0}}\left[\frac{1}{d}\sum_{j=1}^{d}D_{kl}\left({\Pr}_{0}(w^t|w^1\ldots w^{t-1})||{\Pr}_{j}(w^t|w^1\ldots w^{t-1})\right)\right]
\label{eq:whattobound}
\end{align}
Let us focus on a particular choice of $t$ and values $w^1\ldots w^{t-1}$. To
simplify the presentation, we drop the $^t$ superscript from the message
$w^t$, and denote the previous messages $w^1\ldots w^{t-1}$ as $\hat{w}$.
Thus, we consider the quantity
\begin{equation}\label{eq:i0}
\frac{1}{d}\sum_{j=1}^{d}D_{kl}\left({\Pr}_{0}(w|\hat{w})||{\Pr}_{j}(w|\hat{w})\right).
\end{equation}
Recall that $w$ is some function of $\hat{w}$ and a set of $n$ independent
instances received in the current round. Let $\bx_j$ denote the vector of
values at coordinate $j$ across these $n$ instances. Clearly, under $\Pr_j$,
every $\bx_i$ for $i\neq j$ is uniformly distributed on $\{-1,+1\}^n$,
whereas each entry of $\bx_j$ equals $1$ with probability $\frac{1}{2}+\rho$,
and $-1$ otherwise.

We now show that \eqref{eq:i0} can be upper bounded in two different ways,
one bound being $30n\rho b/d$ and the other being $3n\rho^2$. Combining the two, we get
that
\begin{equation}
\frac{1}{d}\sum_{j=1}^{d}D_{kl}\left({\Pr}_{0}(w|\hat{w})||{\Pr}_{j}(w|\hat{w})\right)\leq
\min\left\{30\frac{n\rho b}{d},3n\rho^2 \right\}.
\end{equation}
Plugging this inequality back in \eqref{eq:whattobound}, we validate \eqref{eq:ppr0},
from which the result follows.

\subsubsection*{The $3n\rho^2$ bound}

This bound essentially follows only from the fact that $\bx_j$ is noisy, and not from the
algorithm's information constraints, and is thus easier to obtain.

First, we have by \lemref{lem:decomposition} that for any $w,\hat{w}$,
\[
{\Pr}_{j}(w|\hat{w}) = \sum_{\bx_j}{\Pr}_{0}(w|\hat{w}){\Pr}_{j}(\bx_j|\hat{\bw})
= \sum_{\bx_j}{\Pr}_{0}(w|\hat{w}){\Pr}_{j}(\bx_j)
\]
(this is the same as \eqref{eq:11decomp}, and the justification is the same).

Using this inequality, the definition of relative entropy, and the log-sum inequality, we have
\begin{align*}
\frac{1}{d}&\sum_{j=1}^{d}D_{kl}\left({\Pr}_{0}(w|\hat{w})||{\Pr}_{j}(w|\hat{w})\right)
= \frac{1}{d}\sum_{j=1}^{d}\sum_w {\Pr}_{0}(w|\hat{w})\log\left(\frac{{\Pr}_{0}(w|\hat{w})}
{{\Pr}_{j}(w|\hat{w})}\right)\\
&= \frac{1}{d}\sum_{j=1}^{d}\sum_w {\Pr}_{0}(w|\hat{w})\left(\sum_{\bx_j}{\Pr}_{0}(\bx_j)\right)\log\left(\frac{\sum_{\bx_j}{\Pr}_{0}(w|\bx_j,\hat{w}){\Pr}_{0}(\bx_j)
}{\sum_{\bx_j}{\Pr}_{0}(w|\bx_j,\hat{w}){\Pr}_{j}(\bx_j)}\right)\\
&\leq \frac{1}{d}\sum_{j=1}^{d}\sum_w {\Pr}_{0}(w|\hat{w})\sum_{\bx_j}{\Pr}_{0}(\bx_j)\log\left(\frac{{\Pr}_{0}(w|\bx_j,\hat{w}){\Pr}_{0}(\bx_j)
}{{\Pr}_{0}(w|\bx_j,\hat{w}){\Pr}_{j}(\bx_j)}\right)\\
&=\frac{1}{d}\sum_{j=1}^{d}\sum_w {\Pr}_{0}(w|\hat{w})\sum_{\bx_j}{\Pr}_{0}(\bx_j)\log\left(\frac{{\Pr}_{0}(\bx_j)
}{{\Pr}_{j}(\bx_j)}\right)\\
&=\frac{1}{d}\sum_{j=1}^{d}\sum_{\bx_j}{\Pr}_{0}(\bx_j)\log\left(\frac{{\Pr}_{0}(\bx_j)
}{{\Pr}_{j}(\bx_j)}\right)\\
&=\frac{1}{d}\sum_{j=1}^{d}D_{kl}\left({\Pr}_{0}(\bx_j)||{\Pr}_{j}(\bx_j)\right).
\end{align*}
This relative entropy is between the distribution of $n$ independent
Bernoulli trials with parameter $1/2$, and $n$ independent Bernoulli trials
with parameter $1/2+\rho$. This is easily verified to equal $n$ times the
relative entropy for a single trial, which equals (by definition of relative
entropy)
\[
\frac{1}{2}\log\left(\frac{1/2}{1/2-\rho}\right)+\frac{1}{2}\log\left(\frac{1/2}{1/2+\rho}\right) = -\frac{1}{2}\log\left(1-4\rho^2\right)\leq 8\log(4/3)\rho^2,
\]
where we used the fact that $\rho\leq 1/4n\leq 1/4$, and the inequality $-\log(1-x)\leq 4\log(4/3)x$ for $x\in [0,1/4]$.
Overall, we get that
\[
\frac{1}{d}\sum_{j=1}^{d}D_{kl}\left({\Pr}_{0}(w|\hat{w})||{\Pr}_{j}(w|\hat{w})\right) \leq 8\log(4/3)n\rho^2\leq 3n\rho^2.
\]

\subsubsection*{The $30n\rho b/d$ bound}

To prove this bound, it will be convenient for us to describe the sampling process of
$\bx_j$ in a slightly more complex way, as follows\footnote{We suspect that
this construction can be simplified, but were unable to achieve this without
considerably weakening the bound.}:
\begin{itemize}
\item We let $\bv\in \{0,1\}^n$ be an auxiliary random vector with
    independent entries, where each $v_i=1$ with probability $4\rho$, and
    $0$ otherwise.
\item Under $\Pr_0$ and $\Pr_i$ for $i\neq j$, we assume that $\bx_j$ is
    drawn uniformly from $\{-1,+1\}^n$ regardless of the value of $\bv$.
\item Under $\Pr_j$, we assume that each entry $x_{j,l}$ is independently
    sampled (in a manner depending on $\bv$) as follows:
\begin{itemize}
  \item For each $l$ such that $v_l=1$, we pick $x_{j,l}$ to be $1$ with
      probability $3/4$, and $-1$ otherwise.
  \item For each $l$ such that $v_l=0$, we pick $x_{j,l}$ to be $1$ or
      $-1$ with probability $1/2$.
\end{itemize}
\end{itemize}
Note that this induces the same distribution on $\bx_j$ as before: Each
individual entry $x_{j,l}$ is independent and satisfies
$\Pr_j(x_{j,l}=1)=4\rho*\frac{3}{4}+(1-4\rho)*\frac{1}{2} =
\frac{1}{2}+\rho$.

Having finished with these definitions, we re-write \eqref{eq:i0} as
\[
\frac{1}{d}\sum_{j=1}^{d}D_{kl}\left(\E_{\bv}\left[{\Pr}_{0}(w|\bv,\hat{w})\right]
||\E_{\bv}\left[{\Pr}_{j}(w|\bv,\hat{w})\right]\right).
\]
Since the relative entropy is jointly convex in its arguments, and $\bv$ is a
fixed random variable, we have by Jensen's inequality that this is at most
\[
\E_\bv\left[\frac{1}{d}\sum_{j=1}^{d}D_{kl}\left({\Pr}_{0}(w|\bv,\hat{w})
||{\Pr}_{j}(w|\bv,\hat{w})\right)\right].
\]
Now, note that if $\bv=\mathbf{0}$ (i.e. the zero-vector), then the
distribution of $\bx_1,\ldots,\bx_d$ is the same under both $\Pr_0$ and any
$\Pr_j$. Since $w$ is a function of $\bx_1,\ldots,\bx_d$, it follows that the
distribution of $\bw$ will be the same under both $\Pr_j$ and $\Pr_0$, and
therefore the relative entropy terms will be zero. Hence, we can trivially
re-write the above as
\begin{equation}\label{eq:i05}
\E_\bv\left[\mathbf{1}_{\bv\neq\mathbf{0}}\frac{1}{d}\sum_{j=1}^{d}D_{kl}\left({\Pr}_{0}(w|\bv,\hat{w})
||{\Pr}_{j}(w|\bv,\hat{w})\right)\right].
\end{equation}
where $\mathbf{1}_{\bv\neq\mathbf{0}}$ is an indicator function.

We can now use \lemref{lem:decomposition}, where
$p(\cdot)={\Pr}_{j}(\cdot|\bv,\hat{w})$,$q(\cdot)={\Pr}_{0}(\cdot|\bv,\hat{w})$
and $A_i = \{-1,+1\}^n$ (i.e. the vector of values at a single coordinate $i$
across the $n$ data points). The lemma's conditions are satisfied since
$\bx_i$ for $i\neq j$ has the same distribution under
${\Pr}_{0}(\cdot|\bv,\hat{w})$ and ${\Pr}_{j}(\cdot|\bv,\hat{w})$, and also
$w$ is only a function of $\bx_1\ldots \bx_d$ and $\hat{w}$. Thus, we can
rewrite \eqref{eq:i05} as
\[
\E_{\bv}\left[\mathbf{1}_{\bv\neq\mathbf{0}}\frac{1}{d}\sum_{j=1}^{d}D_{kl}\left({\Pr}_{0}(w|\bv,\hat{w})~\middle|\middle|~
\sum_{\bx_j}{\Pr}_{0}(w|\bx_j,\bv,\hat{w}){\Pr}_{j}(\bx_j|\bv,\hat{w})\right)\right].
\]
Using \lemref{lem:klreverse}, we can reverse the expressions in the relative
entropy term, and upper bound the above by
\begin{equation}\label{eq:prr}
\E_{\bv}\left[\mathbf{1}_{\bv\neq\mathbf{0}}
\frac{1}{d}\sum_{j=1}^{d}\left(\max_{w}\frac{{\Pr}_{0}(w|\bv,\hat{w})}{\sum_{\bx_j}{\Pr}_{0}(w|\bx_j,\bv,\hat{w}){\Pr}_{j}(\bx_j|\bv,\hat{w})}\right)
D_{kl}\left(\sum_{\bx_j}{\Pr}_{0}(w|\bx_j,\bv,\hat{w}){\Pr}_{j}(\bx_j|\bv,\hat{w})~\middle|\middle|~
{\Pr}_{0}(w|\bv,\hat{w})\right)\right].
\end{equation}
The max term equals
\[
\max_{w}\frac{\sum_{\bx_j}{\Pr}_{0}(w|\bx_j,\bv,\hat{w}){\Pr}_{0}(\bx_j|\bv,\hat{w})}{\sum_{\bx_j}{\Pr}_{0}(w|\bx_j,\bv,\hat{w}){\Pr}_{j}(\bx_j|\bv,\hat{w})}
\leq \max_{\bx_j}\frac{{\Pr}_{0}(\bx_j|\bv,\hat{w})}{{\Pr}_{j}(\bx_j|\bv,\hat{w})},
\]
and using Jensen's inequality and the fact that relative entropy is convex in
its arguments, we can upper bound the relative entropy term by
\begin{align*}
&\sum_{\bx_j}{\Pr}_{j}(\bx_j|\bv,\hat{w}) D_{kl}\left({\Pr}_{0}(w|\bx_j,\bv,\hat{w})~\middle|\middle|~
{\Pr}_{0}(w|\bv,\hat{w})\right)\\
&\leq \left(\max_{\bx_j}\frac{{\Pr}_{j}(\bx_j|\bv,\hat{w})}{{\Pr}_{0}(\bx_j|\bv,\hat{w})}\right)\sum_{\bx_j}{\Pr}_{0}(\bx_j|\bv,\hat{w}) D_{kl}\left({\Pr}_{0}(w|\bx_j,\bv,\hat{w})~\middle|\middle|~
{\Pr}_{0}(w|\bv,\hat{w})\right).
\end{align*}
The sum in the expression above equals the mutual information between the
message $w$ and the coordinate vector $\bx_j$ (seen as random variables with
respect to the distribution ${\Pr}_{0}(\cdot|\bv,\hat{w})$). Writing this as
$I_{{\Pr}_{0}(\cdot|\bv,\hat{w})}(w;\bx_j)$, we can thus upper bound
\eqref{eq:prr} by
\begin{align*}
&\E_{\bv}\left[\mathbf{1}_{\bv\neq\mathbf{0}}\frac{1}{d}\sum_{j=1}^{d}\left(\max_{\bx_j}\frac{{\Pr}_{0}(\bx_j|\bv,\hat{w})}{{\Pr}_{j}(\bx_j|\bv,\hat{w})}\right)
\left(\max_{\bx_j}\frac{{\Pr}_{j}(\bx_j|\bv,\hat{w})}{{\Pr}_{0}(\bx_j|\bv,\hat{w})}\right)I_{{\Pr}_{0}(\cdot|\bv,\hat{w})}(w;\bx_j)\right]\\
&\leq \E_{\bv}\left[\mathbf{1}_{\bv\neq\mathbf{0}}\left(\max_{j,\bx_j}\frac{{\Pr}_{0}(\bx_j|\bv,\hat{w})}{{\Pr}_{j}(\bx_j|\bv,\hat{w})}\right)
\left(\max_{j,\bx_j}\frac{{\Pr}_{j}(\bx_j|\bv,\hat{w})}{{\Pr}_{0}(\bx_j|\bv,\hat{w})}\right)\frac{1}{d}\sum_{j=1}^{d}I_{{\Pr}_{0}(\cdot|\bv,\hat{w})}(w;\bx_j)\right].
\end{align*}
Since $\{\bx_j\}_j$ are independent of each other and $w$ contains at most
$b$ bits, we can use the key \lemref{lem:infoexp} to upper bound the above by
\[
\E_{\bv}\left[\mathbf{1}_{\bv\neq\mathbf{0}}\left(\max_{j,\bx_j}\frac{{\Pr}_{0}(\bx_j|\bv,\hat{w})}{{\Pr}_{j}(\bx_j|\bv,\hat{w})}\right)
\left(\max_{j,\bx_j}\frac{{\Pr}_{j}(\bx_j|\bv,\hat{w})}{{\Pr}_{0}(\bx_j|\bv,\hat{w})}\right)\frac{b}{d}\right].
\]
Now, recall that for any $j$, $\bx_j$ refers to a column of $n$ independent
entries, drawn independently of any previous messages $\hat{w}$, where under
${\Pr}_{0}$, each entry $x_{j,i}$ is chosen to be $\pm 1$ with equal
probability, whereas under ${\Pr}_{j}$ each is chosen to be $1$ with
probability $\frac{3}{4}$ if $v_i=1$, and with probability $\frac{1}{2}$ if
$v_i=0$. Therefore, letting $|\bv|$ denote the number of non-zero entries in
$\bv$, we can upper bound the expression above by
\begin{equation}\label{eq:ooo}
\E_{\bv}\left[\mathbf{1}_{\bv\neq\mathbf{0}}
\left(\frac{1/2}{1/4}\right)^{|\bv|}\left(\frac{3/4}{1/2}\right)^{|\bv|}\frac{b}{d}\right]
~=~ \frac{b}{d}\E_{\bv}\left[\mathbf{1}_{\bv\neq\mathbf{0}} 3^{|\bv|}\right],
\end{equation}
To compute the expectation in closed-form, recall that each entry of $\bv$ is
picked independently to be $1$ with probability $4\rho$, and $0$ otherwise.
Therefore,
\begin{align*}
\E_{\bv}\left[\mathbf{1}_{\bv\neq\mathbf{0}}3^{|\bv|}\right] &=
\E_{\bv}\left[3^{|\bv|}-\mathbf{1}_{\bv=\mathbf{0}}\right] \\
&=\prod_{i=1}^{n}\E_{v_i}[3^{v_i}]-\Pr(\bv=\mathbf{0})\\
&=\left(\E_{v_1}[3^{v_1}]\right)^n-\Pr(\bv=\mathbf{0})\\
&=
(4\rho*3+(1-4\rho)*1)^n-(1-4\rho)^n \\
&=
(1+8\rho)^n-(1-4\rho)^n
~\leq~ \exp(8n\rho)-(1-4n\rho),
\end{align*}
where in the last inequality we used the facts that $(1+a/n)^n\leq \exp(a)$
and $(1-a)^n\geq 1-an$. Since we assume $\rho\leq 1/4n$, $8n\rho\leq 2$, so
we can use the inequality $\exp(x)\leq 1+3.2x$, which holds for any $x\in
[0,2]$, and get that the expression above is at most $(1+26n\rho) -(1-4n\rho)
= 30 n \rho$, and therefore \eqref{eq:ooo} is at most $30 n\rho b/d$. This in
turn is an upper bound on \eqref{eq:i0} as required.

\subsection{Proof of \thmref{thm:bandits}}\label{app:bandits}

Let $c_1,c_2$ be positive parameters to be determined later, and assume by
contradiction that our algorithm can guarantee
$\E[\sum_{t=1}^{T}\ell_{t,i_t}-\sum_{t=1}^{T}\ell_{t,j}]< c_1
\min\{T/4,\sqrt{dT/b}\}$ for any distribution and all $j$.

Consider the set of distributions $\Pr_j(\cdot)$ over $\{0,1\}^d$, where each
coordinate is chosen independently and uniformly, except coordinate $j$ which
equals $0$ with probability $\frac{1}{2}+\rho$, where
$\rho=c_2\min\{1/4,\sqrt{d/bT}\}$. Clearly,the coordinate $i$ which minimizes
$\E[\ell_{t,i}]$ is $j$. Moreover, if at round $t$ the learner chooses some
$i_t\neq j$, then $\E[\ell_{t,i_t}-\ell_{t,j}] = \rho =
c_2\min\{1/4,\sqrt{d/bT}\}$. Thus, to have
$\E[\sum_{t=1}^{T}\ell_{t,i_t}-\sum_{t=1}^{T}\ell_{t,j}]<
c_1\min\{T/4,\sqrt{dT/b}\}$ requires that the expected number of rounds where
$i_t\neq j$ is at most $\frac{c_1}{c_2} T$. By Markov's inequality, this
means that the probability of $j$ not being the most-commonly chosen
coordinate is at most $(c_1/c_2)/(1/2) = 2c_1/c_2$. In other words, if we can
guarantee regret smaller than $c_1 \min\{T/4,\sqrt{dT/b}\}$, then we can
detect $j$ with probability at least $1-2c_1/c_2$, simply by taking the most
common coordinate.

However, by\footnote{The theorem discusses the case where the distribution is
over $\{-1,+1\}^d$, and coordinate $j$ has a slight positive bias, but it's
easily seen that the lower bound also holds here where the domain is
$\{0,1\}^d$.} \thmref{thm:full1}, for any $(b,1,T)$ protocol, there is some
$j$ such that the protocol would correctly detect $j$ with probability at
most
\[
\frac{3}{d}+21\sqrt{\frac{Tb}{d}c_2^2\min\left\{\frac{1}{16},\frac{d}{bT}\right\}}
\leq \frac{3}{d}+21c_2.
\]
Therefore, assuming $d>3$, and taking for instance
$c_1=3.7*10^{-4},c_2=5.9*10^{-3}$, we get that the probability of detection
is at most $\frac{3}{4}+21 c_2 < 0.874$, whereas the scheme discussed in the
previous paragraph guarantees detection with probability at least
$1-2c_1/c_2> 0.874$. We have reached a contradiction, hence our initial
hypothesis is false and our algorithm must suffer regret at least $c_1
\min\{T/4,\sqrt{dT/b}\}$.

\subsection{Proof of \thmref{thm:sparsePCA}}\label{app:sparsePCA}

The proof is rather involved, and is composed of several stages. First, we
define a variant of our hide-and-seek problem, which depends on sparse
distributions. We then prove an information-theoretic lower bound on the
achievable performance for this hide-and-seek problem with information
constraints. The bound is similar to \thmref{thm:full}, but without an
explicit dependence on the bias\footnote{Attaining a dependence on $\rho$
seems technically complex for this hide-and-seek problem, but fortunately is
not needed to prove \thmref{thm:sparsePCA}.} $\rho$. We then show how the
lower bound can be strengthened in the specific case of $b$-memory online
protocols. Finally, we use these ingredients in proving
\thmref{thm:sparsePCA}.

We begin by defining the following hide-and-seek problem, which differs from
problem \ref{def:prob1} in that the distribution is supported on sparse
instances. It is again parameterized by a dimension $d$, bias $\rho$, and
sample size $mn$:
\begin{definition}[Hide-and-seek Problem 2]\label{def:prob2}
Consider the set of distributions $\{\Pr_j(\cdot)\}_{j=1}^{d}$ over\\
$\{-\be_i,+\be_i\}_{i=1}^{d}$, defined as
\[
{\Pr}_j(\be_i)=\begin{cases} \frac{1}{2d} & i\neq j\\ \frac{1}{2d}+\frac{\rho}{d} & i=j\end{cases}
\;\;\;
{\Pr}_j(-\be_i)=\begin{cases} \frac{1}{2d} & i\neq j\\ \frac{1}{2d}-\frac{\rho}{d} & i=j\end{cases}.
\]
Given an i.i.d. sample of $mn$ instances generated from $\Pr_j(\cdot)$, where
$j$ is unknown, detect $j$.
\end{definition}
In words, ${\Pr}_j(\cdot)$ corresponds to picking $\pm \be_i$ where $i$ is
chosen uniformly at random, and the sign is chosen uniformly if $i\neq j$,
and positive (resp. negative) with probability $\frac{1}{2}+\rho$ (resp.
$\frac{1}{2}-\rho$) if $i=j$. It is easily verified that this creates sparse
instances with zero-mean coordinates, except coordinate $j$ whose expectation
is $2\rho/d$.

We now present a result similar to \thmref{thm:full} for this new
hide-and-seek problem:
\begin{theorem}\label{thm:sparse}
Consider hide-and-seek problem 2 on $d>1$ coordinates, with some bias \\$\rho
\leq \min\{\frac{1}{27},\frac{1}{9 \log(d)},\frac{d}{14 n}\}$. Then for any
estimate $\tilde{J}$ of the biased coordinate returned by any $(b,n,m)$
protocol, there exists some coordinate $j$ such that
\[
{\Pr}_J(\tilde{J}= j) \leq \frac{3}{d}+11\sqrt{\frac{mb}{d}}.
\]
\end{theorem}
The proof appears in subsection \ref{subsec:thmsparseproof} below, and is
broadly similar to the proof of \thmref{thm:full} (although using a somewhat
different approach).

The theorems above hold for any $(b,n,m)$ protocol, and in particular for
$b$-memory online protocols (since they are a special case of $(b,1,m)$
protocols). However, for $b$-online protocols, the following simple
observation will allow us to further strengthen our results:
\begin{theorem}\label{thm:boundedmem}
    Any $b$-memory online protocol over $m$ instances is also a  $\left(b,\kappa,\left\lfloor\frac{m}{\kappa}\right\rfloor\right)$ protocol for any positive integer $\kappa\leq m$.
\end{theorem}
The proof is immediate: Given a a batch of $\kappa$ instances, we can always
feed the instances one by one to our $b$-memory online protocol, and output
the final message after $\lfloor m/\kappa \rfloor$ such batches are
processed, ignoring any remaining instances. This makes the algorithm a type
of $\left(b,\kappa,\left\lfloor\frac{m}{\kappa}\right\rfloor\right)$
protocol.

As a result, when discussing $b$-memory online protocols for some particular
value of $m$, we can actually apply \thmref{thm:sparse} where we replace
$n,m$ by $\kappa,\lfloor m/\kappa\rfloor$, where $\kappa$ is a free parameter
we can tune to attain the most convenient bound.

With these results at hand, we turn to prove \thmref{thm:sparsePCA}.

The lower bound follows from the concentration of measure assumption on
$\widetilde{x_i x_j}$, and a union bound, which implies that
\[
\Pr\left(\forall i<j,~~ |\widetilde{x_i x_j}-\E[x_i x_j]| < \frac{\tau}{2}\right)
\geq
1-\frac{d(d-1)}{2}2\exp\left(-m\tau^2/6\right)
\geq 1-d^2\exp\left(-m\tau^2/6\right).
\]
If this event occurs, then picking $(\tilde{I},\tilde{J})$ to be the
coordinates with the largest empirical mean would indeed succeed in detecting
$(i^*,j^*)$, since $\E[x_{i^*} x_{j^*}]\geq \E[x_i x_j]+\tau$ for all
$(i,j)\neq (i^*,j^*)$.

The upper bound in the theorem statement follows from a reduction to the
setting discussed in \thmref{thm:sparse}. Let
$\{\Pr_{i^*,j^*}(\cdot)\}_{1\leq i^*<j^*\leq d}$ be a set of distributions
over $d$-dimensional vectors $\bx$, parameterized by coordinate pairs
$(i^*,j^*)$. Each such $\Pr_{i^*,j^*}(\cdot)$ is defined as a distribution
over vectors of the form
$\sqrt{\frac{d}{2}}\left(\sigma_1\be_{i}+\sigma_2\be_{j}\right)$ in the
following way:
\begin{itemize}
    \item $(i,j)$ is picked uniformly at random from $\{(i,j):1\leq i<j\leq
        d\}$
    \item $\sigma_1$ is picked uniformly at random from $\{-1,+1\}$.
    \item If $(i,j)\neq(i^*,j^*)$, $\sigma_2$ is picked uniformly at random
        from $\{-1,+1\}$. If $(i,j)=(i^*,j^*)$, then $\sigma_2$ is chosen
        to equal $\sigma_1$ with probability $\frac{1}{2}+\rho$ (for some
        $\rho\in (0,1/2)$ to be determined later), and $-\sigma_1$
        otherwise.
\end{itemize}
In words, each instance is a $2$-sparse random vector, where the two non-zero
coordinates are chosen at random, and are slightly correlated if and only if
those coordinates are $(i^*,j^*)$.

Let us first verify that any such distribution $\Pr_{i^*,j^*}(\cdot)$ belongs
to the distribution family specified in the theorem:
\begin{enumerate}
\item For any coordinate $k$, $x_k$ is non-zero with probability $2/d$
    (i.e. the probability that either $i$ or $j$ above equal $k$), in which
    case $x_k^2=d/2$. Therefore,  $\E[x_k^2]=1$ for all $k$.
\item When $(i,j)\neq(i^*,j^*)$, then $\sigma_1,\sigma_2$ are uncorrelated,
    hence $\E[x_i x_j]=0$. On the other hand, $\E[x_{i^*} x_{j^*}] =
    \frac{2}{d(d-1)}\left(\left(\frac{1}{2}+\rho\right)\frac{d}{2}+\left(\frac{1}{2}-\rho\right)\left(-\frac{d}{2}\right)\right)=
    \frac{2\rho}{d-1}$. So we can take $\tau=\frac{2\rho}{d-1}$, and have
    that $\E[x_{i^*}x_{j^*}]=\tau$.
\item For any $i<j$, $x_i x_j$ is a random variable which is non-zero with
    probability $2/(d(d-1))$, in which case its magnitude is $d/2$. Thus,
    $\E[(x_i x_j)^2]\leq\frac{d}{2(d-1)}$. Applying Bernstein's inequality,
    if $\widetilde{x_i x_j}$ is the empirical average of $x_i x_j$ over $m$
    i.i.d. instances, then
\[
\Pr\left(\left|\widetilde{x_i x_j}-\E[x_i x_j]\right|\geq \frac{\tau}{2}\right)\leq
2\exp\left(-\frac{m\tau^2}{4\left(\frac{d}{d-1}+\frac{d}{3}\tau\right)}\right). \]
    Since we chose $\tau=\frac{2\rho}{d-1} < \frac{1}{d-1}$, and we assume
    $d\geq 9$, this bound is at most
    \[
    2\exp\left(-\frac{m\tau^2}{\frac{4d}{d-1}\left(1+\frac{1}{3}\right)}\right)
    \leq 2\exp\left(-\frac{m\tau^2}{6}\right).
    \]
\end{enumerate}
Therefore, this distribution satisfies the theorem's conditions.

The crucial observation now is that the problem of detecting $(i^*,j^*)$ is
can be reduced to a hide-and-seek problem as defined in Definition
\ref{def:prob2}. To see why, let us consider the distribution over $d\times
d$ matrices induced by $\bx \bx^\top$, where $\bx$ is sampled according to
$\Pr_{i^*,j^*}(\cdot)$ as described above, and in particular the distribution
on the entries above the main diagonal. It is easily seen to be equivalent to
a distribution which picks one entry $(i,j)$ uniformly at random from
$\{(i,j): 1\leq i<j\leq d\}$, and assigns to it the value
$\left\{-\frac{d}{2},+\frac{d}{2}\right\}$ with equal probability, unless
$(i,j)=(i^*,j^*)$, in which case the positive value is picked with
probability $\frac{1}{2}+\rho$, and the negative value with probability
$\frac{1}{2}-\rho$. This is equivalent to the hide-and-seek problem described
in Definition \ref{def:prob2}, over $\frac{d(d-1)}{2}$ coordinates. Thus, we
can apply \thmref{thm:sparse} for $\frac{d(d-1)}{2}$ coordinates, and get
that if $\rho\leq
\min\left\{\frac{1}{27},\frac{1}{9\log\left(\frac{d(d-1)}{2}\right)},\frac{d(d-1)}{28
n}\right\}$, then for some $(i^*,j^*)$ and any estimator
$(\tilde{I},\tilde{J})$ returned by a $(b,n,m)$ protocol,
\[
{\Pr}_{i^*,j^*}\left((\tilde{I},\tilde{J})=(i^*,j^*)\right) \leq \frac{6}{d(d-1)}+11\sqrt{\frac{2mb}{d(d-1)}}.
\]
Our theorem deals with two types of protocols: $\left(b,d(d-1),\lfloor
\frac{m}{d(d-1)}\rfloor\right)$ protocols, and $b$-memory online protocols
over $m$ instances. In the former case, we can simply plug in $\left\lfloor
\frac{m}{d(d-1)}\right\rfloor, d(d-1)$ instead of $m,n$, while in the latter
case we can still replace $m,n$ by $\left\lfloor
\frac{m}{d(d-1)}\right\rfloor, d(d-1)$ thanks to \thmref{thm:boundedmem}. In
both cases, doing this replacement and choosing
$\rho=\frac{1}{9\log\left(\frac{d(d-1)}{2}\right)}$ (which is justified when
$d\geq 9$, as we assume), we get that
\begin{equation}
{\Pr}_{i^*,j^*}\left((\tilde{I},\tilde{J})=(i^*,j^*)\right) \leq \frac{6}{d(d-1)}+11\sqrt{\frac{2b}{d(d-1)}\left\lfloor \frac{m}{d(d-1)}\right\rfloor}
\leq \Ocal\left(\frac{1}{d^2}+\sqrt{\frac{m}{d^4/b}}\right).\label{eq:prup}
\end{equation}
This implies the upper bound stated in the theorem, and also noting that
\[
\tau = \frac{2\rho}{d-1}=\frac{2}{9(d-1)\log\left(\frac{d(d-1)}{2}\right)}
= \Theta\left(\frac{1}{d\log(d)}\right).
\]

Having finished with the proof of the theorem as stated, we note that it is
possible to extend the construction used here to show performance gaps for
other sample sizes $m$. For example, instead of using a distribution
supported on
\[
\left\{\sqrt{\frac{d}{2}}\left(\sigma_1\be_{i}+\sigma_2\be_{j}\right)\right\}_{1\leq i<j\leq d}
\]
for any pair of coordinates $1\leq i<j\leq d$, we can use a distribution
supported on
\[
\left\{\sqrt{\frac{\lambda}{2}}\left(\sigma_1\be_{i}+\sigma_2\be_{j}\right)\right\}
_{1\leq i<j\leq \lambda}
\]
for some $\lambda \leq d$. By choosing the bias
$\tau=\Theta(1/\lambda\log(\lambda))$, we can show a performance gap (in
detecting the correlated coordinates) in the regime $\frac{\lambda^4}{b}\gg m
\gg \lambda^2\log^2(\lambda)$. This regime exists for $\lambda$ as small as
$\sqrt{b}$ (up to log-factors), in which case we already get performance gaps
when $m$ is roughly linear in the memory $b$.


\subsection{Proof of \thmref{thm:sparse}}\label{subsec:thmsparseproof}

The proof builds on the auxiliary lemmas presented in Appendix
\ref{subsec:auxiliary}. It is broadly similar to the proof of
\thmref{thm:full}, but with a few more technical intricacies (such as
balls-and-bins arguments) to handle the different sampling process.

On top of the distributions ${\Pr}_{j}(\cdot)$ defined in the hide-and-seek
problem (Definition \ref{def:prob2}), we define an additional `reference'
distribution ${\Pr}_{0}(\cdot)$, which corresponds to the instances being
chosen uniformly at random from $\{-\be_i,+\be_i\}_{i=1}^{d}$ (i.e. there is
no biased coordinate).

Let $w^1,\ldots,w^m$ denote the messages computed by the protocol. To show the lower bound, it is enough to prove that
\begin{equation}\label{eq:2ppr0}
\frac{2}{d}\sum_{j=1}^{d}D_{kl}\left({\Pr}_{0}(w^1\ldots w^m)\middle|\middle|{\Pr}_{j}(w^1\ldots w^m)\right) \leq \frac{26 m b}{d},
\end{equation}
since then by applying \lemref{lem:totvar}, we get that for some $j$, ${\Pr}_{j}(\tilde{J}=j) \leq (3/d)+2\sqrt{26 m b/d} < (3/d)+11\sqrt{mb/d}$ as required.

Using the chain rule, the left hand side in \eqref{eq:2ppr0} equals
\begin{align}
&\frac{2}{d}\sum_{j=1}^{d}\sum_{t=1}^{m}\E_{w^1\ldots w^{t-1}\sim {\Pr}_{0}}\left[D_{kl}\left({\Pr}_{0}(w^t|w^1\ldots w^{t-1})||{\Pr}_{j}(w^t|w^1\ldots w^{t-1})\right)\right]\notag\\
&=2\sum_{t=1}^{m}\E_{w^1\ldots w^{t-1}\sim {\Pr}_{0}}\left[\frac{1}{d}\sum_{j=1}^{d}D_{kl}\left({\Pr}_{0}(w^t|w^1\ldots w^{t-1})||{\Pr}_{j}(w^t|w^1\ldots w^{t-1})\right)\right]
\label{eq:2whattobound}
\end{align}
Let us focus on a particular choice of $t$ and values $w^1\ldots w^{t-1}$. To simplify the presentation, we drop the $^t$ superscript from the message $w^t$, and denote the previous messages $w^1\ldots w^{t-1}$ as $\hat{w}$. Thus, we consider the quantity
\begin{equation}\label{eq:2i0}
\frac{1}{d}\sum_{j=1}^{d}D_{kl}\left({\Pr}_{0}(w|\hat{w})||{\Pr}_{j}(w|\hat{w})\right).
\end{equation}
Recall that $w$ is some function of $\hat{w}$ and a set of $n$ instances
received in the current round. Moreover, each instance is non-zero at a
single coordinate, with a value in $\{-1,+1\}$. Thus, given an ordered
sequence of $n$ instances, we can uniquely specify them using vectors
$\bu,\bx_1,\ldots,\bx_d$, where
\begin{itemize}
  \item $\bu\in \{1\ldots d\}^n$, where each $e_i$ indicates what is the
      non-zero coordinate of the $i$-th instance.
  \item Each $\bx_j\in \{-1,+1\}^{|\{i:e_i=j\}|}$ is a (possibly empty) vector of the non-zero values, when those values fell in coordinate $j$.
\end{itemize}
For example, if $d=3$ and the instances are $(-1,0,0);(0,1,0);(0,-1,0)$, then
$\bu=(1,2,2); \bx_1=(-1); \bx_2=(1,-1); \bx_3=\emptyset$. Note that under
both ${\Pr}_{0}(\cdot)$ and ${\Pr}_{j}(\cdot)$, $\bu$ is uniformly
distributed in $\{1\ldots d\}^n$, and $\{\bx_j\}_j$ are mutually independent
conditioned on $\bu$.

With this notation, we can rewrite \eqref{eq:2i0} as
\[
\frac{1}{d}\sum_{j=1}^{d}D_{kl}\left(\E_{\bu}\left[{\Pr}_{0}(w|\bu,\hat{w})\right]||\E_{\bu}\left[{\Pr}_{j}(w|\bu,\hat{w})\right]\right).
\]
Since the relative entropy is jointly convex in its arguments, we have by Jensen's inequality that this is at most
\begin{equation}
\frac{1}{d}\sum_{j=1}^{d}\E_{\bu}\left[D_{kl}\left({\Pr}_{0}(w|\bu,\hat{w})||{\Pr}_{j}(w|\bu,\hat{w})\right)\right]
~=~
\E_{\bu}\left[\frac{1}{d}\sum_{j=1}^{d}D_{kl}\left({\Pr}_{0}(w|\bu,\hat{w})||{\Pr}_{j}(w|\bu,\hat{w})\right)\right].
\label{eq:22i0}
\end{equation}
We now decompose ${\Pr}_{j}(w|\bu,\hat{w})$ using \lemref{lem:decomposition},
where
$p(\cdot)={\Pr}_{j}(\cdot|\bu,\hat{w})$,$q(\cdot)={\Pr}_{0}(\cdot|\bu,\hat{w})$
and each $z_i$ is $\bx_i$. The lemma's conditions are satisfied since the
distribution of $\bx_i$, $i\neq j$ is the same under
${\Pr}_{0}(\cdot|\bu,\hat{w}),{\Pr}_{j}(\cdot|\bu,\hat{w})$, and also $w$ is
only a function of $\bu,\bx_1\ldots \bx_d$ and $\hat{\bw}$. Thus, we can
rewrite \eqref{eq:22i0} as
\[
\E_{\bu}\left[\frac{1}{d}\sum_{j=1}^{d}D_{kl}\left({\Pr}_{0}(w|\bu,\hat{w})~\middle|\middle|~
\sum_{\bx_j}{\Pr}_{j}(\bx_j|\bu,\hat{w}){\Pr}_{0}(w|\bu,\bx_j,\hat{w})\right)\right].
\]
Using \lemref{lem:klreverse}, we can reverse the expressions in the relative entropy term, and upper bound the above by

\begin{align}
&\E_{\bu}\left[\frac{1}{d}\sum_{j=1}^{d}\left(\max_{w}\frac{{\Pr}_{0}(w|\bu,\hat{w})}{\sum_{\bx_j}{\Pr}_{j}(\bx_j|\bu,\hat{w}){\Pr}_{0}(w|\bu,\bx_j,\hat{w})}\right)
\times \right. \notag\\
& \left.D_{kl}\left(\sum_{\bx_j}{\Pr}_{j}(\bx_j|\bu,\hat{w}){\Pr}_{0}(w|\bu,\bx_j,\hat{w})~\middle|\middle|~
{\Pr}_{0}(w|\bu,\hat{w})\right)\right].\label{eq:2prr}
\end{align}
The max term equals
\[
\max_{w}\frac{\sum_{\bx_j}{\Pr}_{0}(\bx_j|\bu,\hat{w}){\Pr}_{0}(w|\bu,\bx_j,\hat{w})}{\sum_{\bx_j}{\Pr}_{j}(\bx_j|\bu,\hat{w}){\Pr}_{0}(w|\bu,\bx_j,\hat{w})}
\leq \max_{\bx_j}\frac{{\Pr}_{0}(\bx_j|\bu,\hat{w})}{{\Pr}_{j}(\bx_j|\bu,\hat{w})},
\]
and using Jensen's inequality and the fact that relative entropy is convex in its arguments, we can upper bound the relative entropy term by
\begin{align*}
&\sum_{\bx_j}{\Pr}_{j}(\bx_j|\bu,\hat{w}) D_{kl}\left({\Pr}_{0}(w|\bu,\bx_j,\hat{w})~\middle|\middle|~
{\Pr}_{0}(w|\bu,\hat{w})\right)\\
&\leq \left(\max_{\bx_j}\frac{{\Pr}_{j}(\bx_j|\bu,\hat{w})}{{\Pr}_{0}(\bx_j|\bu,\hat{w})}\right)\sum_{\bx_j}{\Pr}_{0}(\bx_j|\bu,\hat{w}) D_{kl}\left({\Pr}_{0}(w|\bu,\bx_j,\hat{w})~\middle|\middle|~
{\Pr}_{0}(w|\bu,\hat{w})\right).
\end{align*}
The sum in the expression above equals the mutual information between the
message $w$ and the coordinate vector $\bx_j$ (seen as random variables with
respect to the distribution ${\Pr}_{0}(\cdot|\bu,\hat{w})$). Writing this as
$I_{{\Pr}_{0}(\cdot|\bu,\hat{w})}(w;\bx_j)$, we can thus upper bound
\eqref{eq:2prr} by
\begin{align*}
&\E_{\bu}\left[\frac{1}{d}\sum_{j=1}^{d}\left(\max_{\bx_j}\frac{{\Pr}_{0}(\bx_j|\bu,\hat{w})}{{\Pr}_{j}(\bx_j|\bu,\hat{w})}\right)
\left(\max_{\bx_j}\frac{{\Pr}_{j}(\bx_j|\bu,\hat{w})}{{\Pr}_{0}(\bx_j|\bu,\hat{w})}\right)I_{{\Pr}_{0}(\cdot|\bu,\hat{w})}(w;\bx_j)\right]\\
&\leq \E_{\bu}\left[\left(\max_{j,\bx_j}\frac{{\Pr}_{0}(\bx_j|\bu,\hat{w})}{{\Pr}_{j}(\bx_j|\bu,\hat{w})}\right)
\left(\max_{j,\bx_j}\frac{{\Pr}_{j}(\bx_j|\bu,\hat{w})}{{\Pr}_{0}(\bx_j|\bu,\hat{w})}\right)\frac{1}{d}\sum_{j=1}^{d}I_{{\Pr}_{0}(\cdot|\bu,\hat{w})}(w;\bx_j)
\right].
\end{align*}
Since $\bx_1\ldots \bx_d$ are mutually independent conditioned on $\bu$ and
$\hat{w}$, and also $w$ contains at most $b$ bits, we can use the key
\lemref{lem:infoexp} to upper bound the above by
\[
\E_{\bu}\left[\left(\max_{j,\bx_j}\frac{{\Pr}_{0}(\bx_j|\bu,\hat{w})}{{\Pr}_{j}(\bx_j|\bu,\hat{w})}\right)
\left(\max_{j,\bx_j}\frac{{\Pr}_{j}(\bx_j|\bu,\hat{w})}{{\Pr}_{0}(\bx_j|\bu,\hat{w})}\right)\frac{b}{d}\right].
\]
Now, recall that conditioned on $\bu$, each $\bx_j$ refers to a column of
$|\{i:e_i=j\}|$ i.i.d. entries, drawn independently of any previous messages
$\hat{w}$, where under ${\Pr}_{0}$, each entry is chosen to be $\pm 1$ with
equal probability, whereas under ${\Pr}_{j}$ each is chosen to be $1$ with
probability $\frac{1}{2}+\rho$, and $-1$ with probability $\frac{1}{2}-\rho$.
Therefore, we can upper bound the expression above by
\[
\E_{\bu}\left[\left(\max_j \max\left\{\left(\frac{1/2+\rho}{1/2}\right)^{|\{i:e_i=j\}|},\left(\frac{1/2}{1/2-\rho}\right)^{|\{i:e_i=j\}|}\right\}\right)^2\frac{b}{d}\right].
\]
Since we assume $\rho\leq 1/27$, it's easy to verify that the expression above is at most
\begin{align*}
&\E_{\bu}\left[\left(\max_j \left(1+2.2 \rho\right)^{|\{i:e_i=j\}|}\right)^2\frac{b}{d}\right]
= \E_{\bu}\left[\left(\max_j \left(1+\frac{4.4\rho|\{i:e_i=j\}|}{2|\{i:e_i=j\}|}\right)^{2|\{i:e_i=j\}|}\right)\frac{b}{d}\right]\\
&\leq \E_{\bu}\left[\max_j \exp\left(4.4\rho|\{i:e_i=j\}|\right)\right]\frac{b}{d}
~=~ \E_{\bu}\left[\exp\left(4.4\rho~\max_j|\{i:e_i=j\}|\right)\right]\frac{b}{d}
\end{align*}
Since $\bu$ is uniformly distributed in $\{1\ldots d\}^n$, then
$\max_j|\{i:e_i=j\}$ corresponds to the largest number of balls in a bin when
we randomly throw $n$ balls into $d$ bins. By \lemref{lem:ballsbins}, and
since we assume $\rho \leq
\min\{\frac{1}{27},\frac{1}{9\log(d)},\frac{d}{14n}\}$, it holds that the
expression above is at most $13b/d$. To summarize, this is a valid upper
bound on \eqref{eq:2i0}, i.e. on any individual term inside the expectation
in  \eqref{eq:2whattobound}. Thus, we can upper bound \eqref{eq:2whattobound}
by $26mb/d$. This shows that \eqref{eq:2ppr0} indeed holds, which as
explained earlier implies the required result.

\section{Basic Results in Information Theory}\label{app:info}

The proof of \thmref{thm:full} and \thmref{thm:sparse} makes extensive use of
quantities and basic results from information theory. We briefly review here
the technical results relevant for our paper. A more complete introduction
may be found in  \cite{cover2012elements}. Following the settings considered
in the paper, we will focus only on discrete distributions taking values on a
finite set.

Given a random variable $X$ taking values in a domain $\Xcal$, and having a
distribution function $p(\cdot)$, we define its entropy as
\[
H(X)=\sum_{x\in\Xcal}p(x)\log_2(1/p(x)) = \E_{X}\log_2\left(\frac{1}{p(x)}\right).
\]
Intuitively, this quantity measures the uncertainty in the value of $X$. This
definition can be extended to joint entropy of two (or more) random
variables, e.g. $H(X,Y)=\sum_{x,y}p(x,y)\log_2(1/p(x,y))$, and to conditional
entropy
\[
H(X|Y) = \sum_{y}p(y)\sum_{x}p(x|y)\log_2\left(\frac{1}{p(x|y)}\right).
\]
For a particular value $y$ of $Y$, we have
\[
H(X|Y=y) = \sum_{x}p(x|y)\log_2\left(\frac{1}{p(x|y)}\right)
\]
It is possible to show that $\sum_{j=1}^{n}H(X_i) \geq H(X_1,\ldots,X_n)$,
with equality when $X_1,\ldots,X_n$ are independent. Also, $H(X)\geq H(X|Y)$
(i.e. conditioning can only reduce entropy). Finally, if $X$ is supported on
a discrete set of size $2^b$, then $H(X)$ is at most $b$.

Mutual information $I(X;Y)$ between two random variables $X,Y$ is defined as
\[
I(X;Y) = H(X)-H(X|Y) = H(Y)-H(Y|X) = \sum_{x,y}p(x,y)\log_2\left(\frac{p(x,y)}{p(x)p(y)}\right).
\]
Intuitively, this measures the amount of information each variable carries on
the other one, or in other words, the reduction in uncertainty on one
variable given we know the other. Since entropy is always positive, we
immediately get $I(X;Y)\leq \min\{H(X),H(Y)\}$. As for entropy, one can
define the conditional mutual information between random variables X,Y given
some other random variable $Z$ as
\[
I(X;Y|Z) = \E_{z\sim Z}[I(X;Y|Z=z)]=\sum_z p(z)\sum_{x,y}p(x,y|z)\log_2\left(\frac{p(x,y|z)}{p(x|z)p(y|z)}\right).
\]

Finally, we define the relative entropy (or Kullback-Leibler divergence)
between two distributions $p,q$ on the same set as
\[
D_{kl}(p||q) = \sum_x p(x)\log_2\left(\frac{p(x)}{q(x)}\right).
\]
It is possible to show that relative entropy is always non-negative, and
jointly convex in its two arguments (viewed as vectors in the simplex). It
also satisfies the following chain rule:
\[
D_{kl}(p(x_1\ldots x_n)||q(y_1\ldots y_n) = \sum_{i=1}^{n}\E_{x_1\ldots x_{i-1} \sim p}\left[D_{kl}(p(x_{i}|x_1\ldots x_{i-1})||q(x_{i}|x_1\ldots x_{i-1}))\right].
\]

Also, it is easily verified that
\[
I(X;Y) = \sum_{y}p(y)~D_{kl}(p_{X}(\cdot|y)||p_{X}(\cdot)),
\]
where $p_{X}$ is the distribution of the random variable $X$. In addition, we
will make use of Pinsker's inequality, which upper bounds the so-called total
variation distance of two distributions $p,q$ in terms of the relative
entropy between them:
\[
\sum_x |p(x)-q(x)| \leq \sqrt{2 D_{kl}(p||q)}.
\]

Finally, an important inequality we use in the context of relative entropy calculations is the log-sum inequality. This inequality states that for any nonnegative $a_i,b_i$,
\[
\left(\sum_i a_i\right)\log\frac{\sum_i a_i}{\sum_i b_i} \leq \sum_i a_i \log\frac{a_i}{b_i}.
\]

\end{document}